\documentclass{article}

\PassOptionsToPackage{numbers, sort&compress}{natbib}

\usepackage{natbib}

\usepackage[final]{neurips_2023}

\usepackage[utf8]{inputenc} %
\usepackage[T1]{fontenc}    %
\usepackage{hyperref}       %
\usepackage{url}            %
\usepackage{booktabs}       %
\usepackage{amsfonts}       %
\usepackage{amsthm}
\usepackage{mathtools}
\usepackage{nicefrac}       %
\usepackage{microtype}      %
\usepackage{xcolor}         %
\usepackage{xspace}
\usepackage{xparse}         %
\usepackage{cleveref}       %
\usepackage{stmaryrd}       %
\usepackage{graphicx}       %
\usepackage{floatrow}
\usepackage{bm}
\usepackage{subfigure} 

\usepackage{amsmath}
\usepackage{amssymb}
\usepackage{mathtools}
\usepackage{amsthm}
\usepackage{bm}
\usepackage{nicefrac} 
\usepackage{enumitem}
\usepackage{color}
\usepackage{multirow}
\usepackage{nccmath, amssymb}
\usepackage{bbold}
\usepackage{thmtools,thm-restate}
\usepackage{dblfloatfix} 
\usepackage{cases}

\usepackage{etoolbox,siunitx}
\robustify\bfseries
\graphicspath{{./figures/}}

\usepackage{nicefrac} 
\usepackage{enumitem}
\usepackage{color}
\usepackage{multirow}
\usepackage{nccmath, amssymb}
\usepackage{bbold}
\usepackage{thmtools,thm-restate}
\usepackage{dblfloatfix}

\newtheorem{example}{Example} 
\newtheorem{theorem}{Theorem}
\newtheorem{lemma}{Lemma}
\newtheorem{remark}{Remark}

\newtheorem{proposition}{Proposition}
\newtheorem{corollary}{Corollary}

\DeclareMathOperator*{\esssup}{\ensuremath{\text{\rm ess\,sup}}}

\DeclareMathOperator*{\argmin}{\ensuremath{\text{\rm arg\,min}}}

\DeclareMathOperator*{\range}{\ensuremath{\text{\rm Im}}}

\DeclareMathOperator*{\cl}{\ensuremath{\text{\rm cl}}}

\DeclareMathOperator{\Ker}{\ensuremath{\text{\rm Ker}}}

\DeclareMathOperator{\tr}{\ensuremath{\text{\rm tr}}}

\DeclareMathOperator*{\rank}{\ensuremath{\text{\rm rank}}}

\DeclareMathOperator*{\Spec}{\ensuremath{\text{\rm Sp}}}
\DeclareMathOperator*{\Res}{\ensuremath{\text{\rm Res}}}

\DeclareMathOperator*{\cond}{\ensuremath{\text{\rm cond}}}

\providecommand{\norm}[1]{\lVert#1\rVert}
\providecommand{\SVDr}[1]{[\![#1]\!]_r}
\providecommand{\abs}[1]{\lvert#1\rvert}

\newcommand{\scalarp}[1]{{\langle #1\rangle}}

\newcommand{\R}{\mathbb R}
\newcommand{\C}{\mathbb C}
\newcommand{\N}{\mathbb N}

\newcommand{\EE}{\ensuremath{\mathbb E}}
\newcommand{\PP}{\ensuremath{\mathbb P}}

\newcommand{\Id}{I}
\newcommand{\Data}{\mathcal{D}_n}
\newcommand{\Koop}{{A_{\im}}}  %
\newcommand{\adjKoop}{{A_{\im}^{*}}}  %
\newcommand{\CME}{g_p}
\newcommand{\HKoop}{G_{\RKHS}}  %
\newcommand{\RKoop}{G_\reg}  %
\newcommand{\ERKoop}{\widehat{G}_\reg} 
 
\newcommand{\EEstim}{\widehat{G}}  %
\newcommand{\Estim}{G}  %

\newcommand{\RRR}{\Estim^{\rm RRR}_{r,\reg}}  %
\newcommand{\ERRR}{\EEstim^{\rm RRR}_{r,\reg}}  %

\newcommand{\PCR}{\Estim^{\rm PCR}_{r,\reg}}  %
\newcommand{\EPCR}{\EEstim^{\rm PCR}_{r,\reg}}  %

\newcommand{\ECx}{\widehat{C} } %
\newcommand{\ECy}{\widehat{D}} %
\newcommand{\ECxy}{\widehat{T} }  %
 
\newcommand{\ECreg}{\ECx_\reg}
\newcommand{\TZ}{Z}  %
\newcommand{\EZ}{\widehat{Z}} %
\newcommand{\TS}{S}  %
\newcommand{\ES}{\widehat{S}} %

\newcommand{\TB}{B}  %
\newcommand{\EB}{\widehat{B}} %

\newcommand{\TP}{P}  %
\newcommand{\EP}{\widehat{P}} %

\newcommand{\X}{\mathcal{X}} %
\newcommand{\Risk}{\mathcal{R}} %
\newcommand{\ExRisk}{\mathcal{E}_{\rm HS}} %
\newcommand{\IrRisk}{\mathcal{R}_0} %
\newcommand{\ERisk}{\widehat{\mathcal{R}}} %
\newcommand{\RKHS}{\mathcal{H}} %
\newcommand{\Lii}{L^2_\im(\X)}

\newcommand{\sigalg}{\Sigma_{\X}} %
\newcommand{\im}{\pi} %
\newcommand{\HS}[1]{{\rm{HS}}\left(#1\right)} %
\newcommand{\hnorm}[1]{\norm{#1}_{\rm{HS}}}

\newcommand{\transitionkernel}{p} %
\newcommand{\reg}{\gamma}
\newcommand{\rate}{\varepsilon}

\newcommand{\error}{\mathcal{E}}
\newcommand{\metdist}{\eta}
\newcommand{\emetdist}{\widehat{\eta}}

\newcommand{\rpar}{\alpha}
\newcommand{\spar}{\beta}
\newcommand{\epar}{\tau}
\newcommand{\rcon}{a}
\newcommand{\scon}{b}
\newcommand{\econ}{c_\epar}

\newcommand{\bcon}{c_\RKHS}

\newcommand{\levec}{\widehat{u}}
\newcommand{\revec}{\widehat{v}}
\newcommand{\refun}{\psi}
\newcommand{\erefun}{\widehat{\psi}}

\newcommand{\elefun}{\widehat{\xi}}
\newcommand{\kefun}{f}
\newcommand{\ekefun}{\widehat{\kefun}}
\newcommand{\eval}{\lambda}
\newcommand{\keval}{\mu}
\newcommand{\eeval}{\widehat{\eval}}
\newcommand{\gap}{\ensuremath{\text{\rm gap}}}

\newcommand{\one}{\mathbb 1}

\newcommand{\spH}{\mathcal{H}}
\newcommand{\spG}{\mathcal{G}}
\newcommand{\fH}{\phi}
\newcommand{\fG}{\psi}
\newcommand{\HSr}{{\rm{B}}_r({\RKHS})}
\newcommand{\Cx}{C}
\newcommand{\Creg}{C_\reg}

\newcommand{\Cxy}{T}

\newcommand{\Kx}{K}
\newcommand{\Kreg}{K_{\reg}}
\newcommand{\Ky}{L}
\newcommand{\Kyx}{M}
\newcommand{\ranke}{\mathbf{r}}

\usepackage[textsize=tiny]{todonotes}

\newcommand{\vladi}[2][]{\todo[color=red!20,#1]{{\bf VK:} #2}}

\newcommand{\VK}[1]{{#1}}

\newcommand{\KL}[1]{\textcolor{green}{#1}}

\title{Sharp Spectral Rates for Koopman Operator Learning}

\author{Vladimir R. Kostic \\ {Istituto Italiano di Tecnologia} \\ {University of Novi Sad} \\ {\tt \small vladimir.kostic@iit.it} \And
Karim Lounici \\
CMAP-Ecole Polytechnique \\
{\tt \small karim.lounici@polytechnique.edu} \And
Pietro Novelli \\ {Istituto Italiano di Tecnologia} \\ {\tt \small pietro.novelli@iit.it} \And 
Massimiliano Pontil \\ {Istituto Italiano di Tecnologia} \\ {University College London} \\ {\tt \small massimiliano.pontil@iit.it}
}

\begin{document}

\maketitle

\begin{abstract}
Nonlinear dynamical systems can be handily described by the associated Koopman operator, whose action evolves every observable of the system forward in time. Learning the Koopman operator and its spectral decomposition from data is enabled by a number of algorithms. In this work we present for the first time non-asymptotic learning bounds for the Koopman eigenvalues and eigenfunctions. 
We focus on time-reversal-invariant stochastic dynamical systems,  
including the important example of Langevin dynamics. 
We analyze two popular estimators: Extended Dynamic Mode Decomposition (EDMD) and Reduced Rank Regression (RRR). Our results critically hinge on novel {minimax} estimation bounds for the operator norm error, that may be of independent interest. Our spectral learning bounds are driven by the simultaneous control of the operator norm error and a novel metric distortion functional of the estimated eigenfunctions. The bounds indicates that both EDMD and RRR have similar variance, but EDMD suffers from a larger bias which might be detrimental to its learning rate. Our results shed new light on the emergence of spurious eigenvalues, an issue which is well known empirically. Numerical experiments illustrate the implications of the bounds in practice.
\end{abstract}

\section{Introduction}\label{sec:intro}

Recently, researchers have emphasized the utmost importance of developing physically-informed machine learning models that prioritize interpretability and foster physical insight and intuition, see for example \cite{karniadakis2021physics} and references therein. One technique highlighted in these works is the Koopman operator regression framework to learn and interpret nonlinear dynamical systems see, e.g. \cite{Brunton2022,Kutz2016} and references therein. A key component of this approach is the Koopman Mode Decomposition (KMD), which decomposes complex dynamical systems into simpler, coherent structures. \VK{When ordinary least squares are used to learn Koopman operator from data, estimated KMD is known as the Dynamic Mode Decomposition (DMD)~\cite{Rowley2009}.} Koopman operator estimators and their modal decomposition find many applications, including fluid dynamics, molecular kinetics and robotics~\cite{Folkestad2021, Bruder2021}.%

The Koopman operator returns the expected value of observables of the system in the future given the present, and one relies on estimators of this operator to in turn estimate its spectral decomposition that leads to the estimation of KMD.  
Our goal is to study the statistical properties of the eigenvalues and eigenfunctions of the Koopman operator estimators via two mainstream algorithms: Principal Component Regression~(PCR) and Reduced Rank Regression~(RRR) studied in \cite{Heas2021, Kostic2022}. PCR  encompasses as particular cases the popular Extended Dynamic Mode Decomposition (EDMD), which is the de-facto estimator in the data-driven dynamical system literature \citep[see][and references therein]{OWilliams2015, Kutz2016}. Both PCR and RRR are kernel-based algorithms that, given a dataset of observations of the dynamical system, implement a strategy to approximate the action of the Koopman operator on a reproducing kernel Hilbert space (RKHS)~\cite{aron1950,Steinwart2008}.

We present for the first time non-asymptotic learning bounds on the distance between the Koopman eigenvalues and eigenfunctions and those estimated by either PCR or RRR. We show that the eigenvalues produced by such algorithms are {\em biased} estimators of the true Koopman eigenvalues, with PCR incurring larger bias. Our results critically hinge on novel estimation bounds for the operator norm error, that may be of independent interest, leading to minimax optimal bounds for finite-rank Koopman operators. Moreover, we introduce the novel notion of metric distortion, which characterize how the norm of eigenfunctions vary when moving from the RKHS in which learning takes place to the underlying ambient space where the Koopman operator is properly defined. We show that both the operator norm error and metric distortion are needed in order to estimate the operator spectra and our bounds can be used to explain the well-known spuriousness phenomena in eigenvalue estimation \cite{Colbrook2021}, namely, the scenario in which the estimated eigenvalues are not related to the true ones, despite small operator norm error.

{\bf Contributions and Organization.~} We make the following contributions: {\bf i)} We introduce the notion of metric distortion and show that it has to be used alongside the operator norm error to derive Koopman spectra estimation error bounds (Theorem \ref{thm:spectral_perturbation}); {\bf ii)} We establish the first sharp estimation bound for the operator norm error (Theorem \ref{thm:error_bound}); {\bf iii)} We establish spectral learning rates (Thms. \ref{thm:spectral_uniform_main} and \ref{thm:spectral_main}) for both PCR and RRR; {\bf iv)} We propose how to use the results entailed by Theorem~\ref{thm:spectral_main} to detect the presence of spurious eigenvalues from data. 

The paper is organized as follows. In Section~\ref{sec:background} we recall the notion of Koopman operator, its spectral decomposition, and review PCR and RRR estimators. Section~\ref{sec:main_short} describes the estimation problem and outline our main results. Section~\ref{sec:approach} presents our approach to bound eigenvalue and eigenvector estimation errors. Section~\ref{sec:error}  gives sharp upper bounds for the operator norm error.  Section~\ref{sec:spectral_rates} presents our spectral learning bounds. Finally, Section~\ref{sec:exp} 
illustrates the implications of the bounds in practice, and 
is designed to provide practitioners with the tools to benchmark the performance of algorithms in real scenarios. 

\section{Background}\label{sec:background}

 \textbf{Dynamical Systems and Koopman Operator.} In this work we study Markovian dynamical systems, that is collections of random variables $\{X_{t} \colon t \in \N\}$, where $X_{t}$ represents the {\em state} at time $t$, taking values in some space $\X$. 
We focus on time-homogeneous (i.e. autonomous) systems hosting an invariant measure $\im$ for which the {\em Koopman operator} \cite{Lasota1994,Kostic2022} 
\begin{equation}\label{eq:koopman_informal}
    (\Koop f)(x) := \EE[f(X_{t + 1})| X_{t} = x],\quad x \in \X
\end{equation}
is a well defined bounded linear operator on $\Lii$, the space of square integrable functions on $\X$ relative to 
measure $\pi$. {In the field of stochastic processes,~\eqref{eq:koopman_informal} is also known as the {\em transfer operator} and returns the expected value of $f$ in the future given the present.} This operator is is self-adjoint (i.e. $\Koop \,{=}\, \adjKoop$) whenever dynamics is time-reversal invariant w.r.t. $\im$, which is satisfied by many  stochastic processes in the physical sciences.

\begin{example}[Langevin Dynamics]\label{ex:langevin}
    Let $\X = \R^{d}$ and let $\beta\, {>} \,0$. The (overdamped) Langevin equation driven by a potential  $U:\R^{d} \to \R$ is given by
   $dX_{t} = -\nabla U(X_t)dt + \sqrt{2\beta^{-1}}dW_t$, where $W_{t}$ is a Wiener process. The invariant measure of this process is the {\em Boltzman distribution} $\pi(dx) \propto e^{-\beta U(x)} dx$, and the associated Koopman operator is self-adjoint.
    \end{example}

The Langevin equation models a wealth of phenomena, such as the evolution of chemical and biological systems at thermal equilibrium~\cite{Davidchack2015}, the mechanism regulating cell size in bacteria~\cite{Amir2014}, chemical reactions~\cite{Kramers1940}, the dynamics of synapses~\cite{Ventriglia2000, Choquet2013}, stock market fluctuations~\cite{Bouchaud1998} and many more. Furthermore, when $U(x) = \theta\norm{x}^{2}/2$ ($\theta > 0$), the Langevin equation reduces to the celebrated Ornstein–Uhlenbeck process~\citep[Chapter 6]{Pavliotis2014}. 
    
The operator~\eqref{eq:koopman_informal} evolves every observable of the system forward in time. Since it is bounded and linear, it admits a {\em spectral decomposition}, which plays a central role in the analysis and interpretation of the dynamical system~\cite{Kutz2016}, as well as (nonlinear) control \cite{AM2017}.  As in \cite{Wu2019}, to study the spectral decomposition we further assume that $\Koop$ is a {\em compact} operator, which rules out the presence of continuous and residual spectrum components and leads to
\begin{equation}
    \Koop =  \textstyle{\sum_{i\in\N}}\,\keval_i\, \kefun_i\otimes\kefun_i,
\label{eq:specK}
\end{equation}
where $(\keval_{i},\kefun_{i})_{i\in\N}\subseteq\R\,\times\,\Lii$ are Koopman eigenpairs, i.e. $\Koop\kefun_{i}=\keval_{i}\,\kefun_{i}$. Moreover, $\lim_{i \rightarrow \infty} \mu_i = 0$ and $\{\kefun_{i}\}_{i\in\N}$ form a complete orthonormal system of $\Lii$. In the context of molecular dynamics, the leading eigenvalues 
and their eigenfunctions are key in the study of long-term dynamics and so-called {\it meta-stable} states \citep[see, e.g.,][]{tuckerman2010statistical}.

 \textbf{Koopman Operator Regression in RKHS.}
Throughout the paper we let $\RKHS$ be an RKHS and let $k:\X\times\X \to \R$ be the associated kernel function. We let $\phi:\X \to \RKHS$ be a {\em feature map}~\cite{Steinwart2008} such that $k(x,x^\prime) = \scalarp{\phi(x), \phi(x^\prime)}$ for all $x, x^\prime \in \X$. We consider RKHSs satisfying $\RKHS \subset \Lii$~\cite[Chapter 4.3]{Steinwart2008}, so that PCR and RRR approximate $\Koop:\Lii \to \Lii$ with an operator $\Estim:\RKHS \to \RKHS$. Notice that despite $\RKHS \subset \Lii$, the two spaces have different metric structures, that is for all $f,g \in \RKHS \subset \Lii$, one in general has $\scalarp{f, g}_{\RKHS} \neq \scalarp{f,g}_{\Lii}$. In order to handle this ambiguity, we introduce the {\em injection operator} $\TS:\RKHS \to \Lii$ such that for all $f \in \RKHS$, the object $\TS f$ is the element of $\Lii$ which is pointwise equal to $f \in \RKHS$, but endowed with the appropriate $\Lii$ norm. With this in mind, the Koopman operator restricted to $\RKHS$ is simply $\Koop\TS$, which is then estimated by $\TS\Estim$ for some $\Estim \in \HS{\RKHS}$. We will measure the operator norm error, $\norm{\Koop\TS-\TS\Estim}$. 
This is in contrast to the %
more frequently used Hilbert-Schmidt (HS) norm. 

Koopman operator regression estimators are supervised learning algorithms to learn the Koopman operator, in which input and output data are consecutive states of the system $(X_{t}, X_{t + 1})$ for some $t\in \N$. Since the Markov process is time-homogeneous and stationary, the joint probability distribution of $(X_{t}, X_{t + 1})$ is the same for every $t\in\N$ and we denote it by $\rho$. Furthermore, stationarity also implies that $X_{t} \sim \im$ for all $t \in \N$. 
Given a dataset\footnote{For simplicity we consider the i.i.d. setting, however our forthcoming analysis is directly applicable to sample trajectories following \citep{Kostic2022}.} $\Data := (x_{i}, y_{i})_{i = 1}^{n}$ of consecutive states, PCR and RRR are two different strategies to minimize, under a fixed-rank constraint, the  mean square error
\begin{equation}\label{eq:empirical_risk}
    {\ERisk}(\Estim) := \textstyle{\tfrac{1}{n}\sum_{i \in[n]}} \norm{\phi(y_{i}) - \Estim^{*}\phi(x_i)}^{2},
\end{equation} 
where $\Estim \in \HS{\RKHS}$, the space of Hilbert-Schmidt operator acting on $\RKHS$. PCR and RRR estimators are expressed as functions of the {\em input} and {\em cross} empirical covariances, defined 
respectively as
$$\ECx \,{=} \,\textstyle{\tfrac{1}{n}\sum_{i \in[n]}}\, \phi(x_{i}){\otimes} \phi(x_{i}), ~~{\rm and}~~\ECxy\,{ =} \,\textstyle{\tfrac{1}{n}\sum_{i \in[n]}}\, \phi(x_{i}){\otimes }\phi(y_{i}).$$ Likewise, the population risk is
\(
    \Risk(G) = \EE_{(X,Y)\sim\rho} \norm{\phi(Y) - G^{*}\phi(X)}^{2}
\), and the population covariance and cross convariance are
$\Cx \,{=} \,  \EE_{X \sim \im}\phi(X){\otimes} \phi(X)$, and $\Cxy \,{=} \, \EE_{(X, Y) \sim \rho}\phi(X){\otimes} \phi(Y)$, respectively. We note that by the {\em reproducing kernel property} 
one finds 
that $\Cx = \TS^{*}\TS$; see e.g.~\cite{Steinwart2008}.

\textbf{Two Important Estimators.} We next briefly recall two operator regression estimators that we study in this paper. The {\em Principal Component Regression} (PCR) estimator works by first projecting the input data into the $r$-dimensional principal subspace of the covariance matrix $\ECx$, and then ordinary least squares are solved for such projected data, yielding the estimator~\citep[see e.g.][]{Kostic2022} 
\begin{equation}\label{eq:empirical_PCR_estimator}
    \EPCR = \SVDr{\ECreg^{-1}}\ECxy.
\end{equation}
Here $\ECx_{\reg} := \ECx + \reg\Id_{\RKHS}$ and $\SVDr{\cdot}$ denotes the $r$-truncated SVD. The population counterpart is  $\PCR = \SVDr{\Creg^{-1}}\Cxy$, where $\Cx_{\reg} := \Cx + \reg\Id_{\RKHS}$. Note, however, that the empirical PCR estimator does {\em not} minimize the empirical risk~\eqref{eq:empirical_risk} under the low-rank constraint.  

The {\em Reduced Rank Regression (RRR)} algorithm, in contrast, is the {\em exact} minimizer of~\eqref{eq:empirical_risk} under fixed rank constraint. Specifically, RRR is defined as
$\ERRR := \argmin\{\hat{\Risk}(\Estim) + \reg \hnorm{\Estim}^{2} : \Estim \in \HSr \}$,
where the {\em regularization} term $\reg \hnorm{\Estim}^{2}$ is added to ensure stability, and $\HSr$ denotes the set of bounded operators on $\RKHS$ that have rank at most $r$. The closed form solution of the empirical RRR estimator is~\cite{Kostic2022} 
\begin{equation}\label{eq:empirical_RRR_estimator}
   \ERRR = \ECx_{\reg}^{-1/2}\SVDr{\ECx_{\reg}^{-1/2}\ECxy},
\end{equation}
while the population counterpart is given by $\RRR = \Cx_{\reg}^{-1/2}\SVDr{\Cx_{\reg}^{-1/2}\Cxy}$.

Once either the PCR or RRR estimators are fitted, their spectral decomposition is a proxy for the spectral decomposition of the Koopman operator $\Koop$. Theorem 2 in~\cite{Kostic2022} shows how such a decomposition can be calculated via the kernel trick for both $\EPCR$ and $\ERRR$.

\begin{figure}[t!]
    \centering
\includegraphics[width=0.78\textwidth]{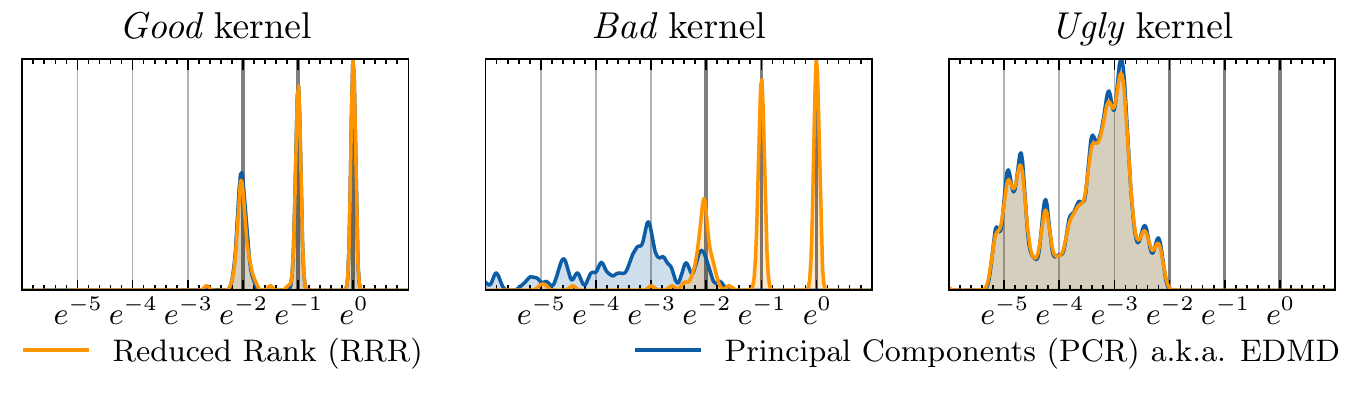}
\vspace{-.3truecm}
    \caption{PCR vs. RRR in estimating the largest eigenvalues of the 1D Ornstein–Uhlenbeck process with three different kernels over 50 independent trials. Vertical lines correspond to Koopman eigenvalues. The \textit{good} kernel is such that its $\RKHS$ corresponds to the leading eigenspace of the Koopman operator, while the other two are spans of scaled and permuted eigenfunctions for which the distortion with respect to the original metric structure of $\Koop$ introduce slow (\textit{bad} kernel) and fast (\textit{ugly} kernel) spectral decay of the covariance.}
 \label{fig:good_bad_ugly}
\end{figure}

\section{The Problem and Main Result in a Nutshell}\label{sec:main_short}
In this section we introduce the spectral estimation problem, outline our main results in a distilled form, and discuss some important implications. 
Recall the definition of Koopman operator \eqref{eq:koopman_informal} and its spectral decomposition \eqref{eq:specK}. 
Given a rank $r$ estimator $\EEstim \,{\in}\, \HSr$ of $\Koop$, we let  $(\eeval_i,\erefun_i)_{i=1}^r$ be its spectral decomposition, satisfying $\EEstim\erefun_i \,{=}\, \eeval_i\,\erefun_i$. We aim to study how well a nonzero eigenvalue $\eeval_i$ of $\EEstim$ estimates its {\em closest} Koopman eigenvalue $\keval_{j(i)}$, where %
\begin{equation}
    j(i) = {\rm argmin}_{j \in \N} |\eeval_i - \mu_j|.
    \label{eq:jofi}
\end{equation}
Moreover we wish to compare $\erefun_i$ with the corresponding true Koopman eigenfunction. To this end, we embed $\erefun_i$ in $\Lii$ by means of the operator $\TS$ and define the normalized estimated eigenfunction
\begin{equation}
\ekefun_i= \TS\erefun_i \, / \, \norm{\TS\erefun_i}.
\end{equation}
One of the key quantities studied in this work is the eigenvalue estimation error
\(
\abs{\eeval_i-\keval_{j(i)}},~i \in [r].
\)
Recalling that $\Koop$ is compact and self-adjoint, the classical Davis-Kahan result~\cite{DK1970} implies that the eigenvalue estimation error $\abs{\eeval_i-\keval_{j(i)}}$ also bounds the quality of the eigenfunction approximation as
\begin{equation}\label{eq:eigfunc_approximation_error}
\norm{\ekefun_i-\kefun_{j(i)}}^2\leq \frac{2\abs{\eeval_i-\keval_{j(i)}}}{[\gap_{j(i)}(\Koop) - \abs{\eeval_i-\keval_{j(i)}}]_{+}}
\end{equation}
where $\gap_{j}(\Koop){=}\min_{\ell\neq j}\abs{\keval_\ell{-}\keval_j}$ is the distance between $\keval_{j}$ and its closest Koopman eigenvalue.  
Let $\sigma_j(\cdot)$ denotes the $j$-th singular value of an operator. To give a flavour of our results, here we report spectral bounds for the Gaussian kernel. In this case, Theorem~\ref{thm:spectral_uniform_main} below gives a high probability bound on the estimation error $\abs{\eeval_i-\keval_{j(i)}}$, that is of order
\[
\mathcal{O}\bigg(\dfrac{\sigma_{r+1}(\Koop\TS) }{\sigma_{r}(\Koop\TS)} + \frac{1}{\sqrt{n}}\bigg)\,\,\, {\rm for}\,\,\, \ERRR, \,\,\, {\rm and}\,\, \,\mathcal{O}\left(\dfrac{\sigma_{r+1}(\TS)}{[\sigma_{r}(\Koop\TS) - \sigma_{r+1}(\TS)]_+} + \frac{1}{\sqrt{n}}\right)\,\,\,{\rm for}\,\,\, \EPCR.
\]
If the Koopman operator has finite rank then $\sigma_{r+1}(A_\pi S) \,{=}\, 0$,  
the RRR estimator is unbiased, and its error goes to zero {at the rate} $1/\sqrt{n}$.
{Otherwise}, recalling that $\sigma_{r+1}(\TS)$ is the square root of the $(r \,{+}\, 1)$-th eigenvalue of the kernel operator~\citep[Chapter 4.5]{Steinwart2008}, if 
$\RKHS$ is infinite dimensional 
$\sigma_{r+1}(\TS) > 0$, i.e. PCR has a strictly positive bias. In general, the presence of a bias in the estimated eigenvalues may result in the appearance of \textit{spurious eigenvalues}.
This phenomenon for PCR is well documented in practice, see e.g.~\cite{Lewin2009, Colbrook2019, Colbrook2021, Kutz2016}. %
In Figure~\ref{fig:good_bad_ugly} we illustrate such an effect on a simple dynamical system discussed both in Example~\ref{ex:OU} and in Section~\ref{sec:exp}.

\section{Approach}\label{sec:approach}

The core of our analysis is 
Theorem~\ref{thm:spectral_perturbation}. %
It reveals that in order to derive spectral estimation bounds for the Koopman operator, it is not enough to study the excess risk in the HS norm. Indeed, our spectral bounds are determined by both the \textit{operator norm error} of the Koopman estimator  
\begin{equation}\label{eq:error}
\error(\EEstim):=\norm{\Koop\TS-\TS \EEstim},\;\EEstim\in\HS{\RKHS}
\end{equation}
and the \textit{metric distortion} between $\RKHS$ and $\Lii$, 
\begin{equation}
\metdist(h):=\norm{h}\, / \, \norm{\TS h},\,\,\, h\in\RKHS.
    \label{eq:MD}
\end{equation}
\VK{Note that since $\TS h\in\Lii$ is just an equivalence class of a function $h$, $\norm{\TS h}$ is simply $\Lii$-norm of $h$, and hence the metric distortion can be written, with a slight abuse of notation, as $\metdist(h):=\norm{h}_{\RKHS} / \norm{h}_{\Lii}$.}
While the (HS norm) error was studied before~\citep[see][and references therein]{Li2022}, little is know about operator norm error bounds. Moreover, the metric distortion is, to the best of our knowledge, a novel quantity in the spectral analysis of Koopman operator. 

\begin{restatable}{theorem}{thmSpPert}\label{thm:spectral_perturbation} Let $\Koop$ be a self-adjoint compact operator and let $r\in\N$. Then, 
for every empirical estimator $\EEstim\in\HSr$ and every $i\in[r]$
\vspace{-0.2truecm}
\begin{equation}
\label{eq:bound_eval}
\abs{\eeval_i-\keval_{{j(i)}}}  \leq \metdist(\erefun_i)\,\error(\EEstim), \quad {\rm and}\quad
\norm{\ekefun_{i} - \kefun_{{j(i)}}}^2 \leq \frac{2\metdist(\erefun_i)\,\error(\EEstim) }{[\gap_{{j(i)}}(\Koop) - \metdist(\erefun_i)\,\error(\EEstim)]_{+}}.
\end{equation}
\end{restatable}
\vspace{-0.4truecm}
\begin{proof}[Proof Sketch.]
First, note that for compact self-adjoint operators $\abs{\eeval_i-\keval_{{j(i)}}}\leq\norm{(\Koop-\eeval_i\,I)^{-1}}^{-1}$. So, following the reasoning of~\citep[][Theorem~1]{Kostic2022} and observing that $\norm{(\Koop\TS-\TS\EEstim)\erefun_i} / \norm{\TS\erefun_i}\leq\error(\EEstim)\metdist(\erefun_i)$ gives the right hand side of the first equation in \eqref{eq:bound_eval}. Next, since additionally $\norm{(\Koop\TS-\TS\EEstim)\erefun_i} / \norm{\TS\erefun_i}\leq\error(\EEstim)\metdist(\erefun_i)$, we can apply the Davis-Kahan spectral perturbation result for compact self-adjoint operators (Proposition~\ref{prop:davis_kahan}, Appendix~\ref{app:approach}) to bound $\sin(\widehat{\theta})$, where $\widehat{\theta}_i:=\sphericalangle(\ekefun_i,\kefun_{{j(i)}})$. 
The claim then follows since   $\norm{\ekefun_i\,{-}\,\kefun_{{j(i)}}}^2\,{\leq}\, 2(1\,{-}\,\cos(\widehat{\theta_i}))\,{\leq}\, 2\sin(\widehat{\theta_i})$. The full proof can be found in Appendix~\ref{app:approach}. 
\end{proof}
Note that the error~\eqref{eq:error}, at least for universal kernels, can be made arbitrary small, see \citep[Proposition~1]{Kostic2022}. 
Still, the metric distortion may dominate the error and, since the bound~\eqref{eq:bound_eval} is tight, one may have that the operator is well estimated in norm, but the estimated eigenpairs are far from the true ones. This phenomenon is at the origin of {\em spurious eigenvalues}. The proposed way to detect them {for \emph{deterministic} systems in~\cite{Colbrook2021} is to check if eigenvalue equations are satisfied empirically, which, however, is not useful for \emph{stochastic} systems,
see Rem.~\ref{rem:spurious} of Appendix~\ref{app:approach}. }

Spuriousness may also originate from poor conditioning of the true eigenvalues, i.e. when the angle between true left and right eigenfunctions is small. Here, however, we assume $\Koop=\adjKoop$, so that we restrict ourselves to the case in which the \textit{only source of spuriousness is due to the learning method}.

While we defer the discussion of the operator norm error to the next section, the following result bound 
the metric distortion; the proof can be found in 
Appendix~\ref{app:approach}. 
\begin{restatable}{proposition}{propMetDist}\label{prop:metric_dist}
Let $\EEstim\,{\in}\,\HSr$. For all $i\in[r]$ the metric distortion of $\erefun_i$ can be tightly bounded as
\begin{equation}\label{eq:metric_dist_bound}
1\,\,/\sqrt{\norm{\Cx}} \,\leq\, \metdist(\erefun_i)\, \leq\, \min(\abs{\eeval_i}\cond(\eeval_i), \norm{\EEstim}) \, / \, \sigma_{\min}^{+}(\TS\EEstim),
\end{equation}
where  $\cond(\eeval_i):=\norm{\elefun_i}\norm{\erefun_i} / \abs{\scalarp{\erefun_i,\elefun_i}}$ is the condition number of $\eeval_i$, and $\elefun_i$ is its left eigenfunction.%
\end{restatable}
The upper bound \eqref{eq:metric_dist_bound}  depends on the estimator's eigenvalues and their conditioning.
Notice that 
while the true eigenvalues of $\Koop$ have condition number one, the conditioning of the estimated ones depends on the choice of the kernel. Moreover, the upper bound can be controlled by tuning the estimator rank. Since the bound  is tight (see Rem.~\ref{rem:metdist_tight} in Appendix~\ref{app:approach}), %
the metric distortion can grow with the rank of the estimator,
further motivating the use of low-rank estimators of $\Koop$ in practice, see~\cite{Kutz2016}. 

We end this section by introducing an empirical estimator of 
the metric distortion $\metdist(\erefun_i)$, given by 
\begin{equation}\label{eq:emp_met_dist_main}
\emetdist_i:= \norm{\erefun_i}\, /\, \sqrt{\scalarp{\ECx \erefun_i,\erefun_i}}.
\end{equation}
Proposition~\ref{prop:emp_met_dist} of Appendix~\ref{app:approach} shows that $\emetdist_i$ can be efficiently computed and report upper bounds for concentration around its mean. {The empirical metric distortion~\eqref{eq:emp_met_dist_main}, used in conjunction with the spectral bounds in Theorem~\ref{thm:spectral_main} below, provides a proxy to assess the reliability of the PCR and RRR estimators and can be successfully used as novel model selection criterion.%
We refer the reader to the second and third experiment in Section~\ref{sec:exp} for concrete use-cases.}

\section{Controlling the Operator Norm Error}\label{sec:error} 
The HS norm error of the PCR estimator was already studied, either in the "well-specified" setting, \VK{i.e. when there exists $\HKoop\in\HS{\RKHS}$ such that $\Koop\TS=\TS\HKoop$, i.e. $\HKoop$ is $\im$-a.e. Koopman operator}~\citep[Theorem B.10]{Ciliberto2020}. On the other hand,  KRR estimator is studied also in the "misspecified setting"~\cite{Li2022}. But, up to our knowledge, the operator norm error has not yet been studied. To analyse these learning rates, we make 
the following assumptions:
\vspace{-.015truecm}
\begin{enumerate}[label={\rm \textbf{(RC)}},leftmargin=7ex,nolistsep]
\item\label{eq:RC} \emph{Regularity of $\Koop$.} For some $\rpar\in(0,2]$ there exists $\rcon>0$ such that
$\Cxy \Cxy^* \preceq \rcon^2 \Cx^{1+\alpha}$;
\end{enumerate}
\begin{enumerate}[label={\rm \textbf{(BK)}},leftmargin=7ex,nolistsep]
\item\label{eq:BK} \emph{Boundedness.}  There exists $\bcon\,{>}\,0$ such that $\displaystyle{\esssup_{x\sim\im}}\norm{\phi(x)}^2\leq \bcon$, i.e.
$\phi\in L^\infty_\im(\X,\RKHS)$;
\end{enumerate}
\vspace{-.195truecm}
\begin{enumerate}[label={\rm \textbf{(SD)}},leftmargin=7ex,nolistsep]
\item\label{eq:SD} \emph{Spectral Decay.} There exists $\spar\,{\in}\,(0,1]$ and 
$\scon\,{>}\,0$ such that
$\eval_j(\Cx)\,{\leq}\,\scon\,j^{-1/\spar}$, for all $j\in J$.
\end{enumerate}
While we keep assumptions \ref{eq:BK} and \ref{eq:SD} as in \cite{Fischer2020,Li2022}, 
assumption \ref{eq:RC} is, up to our knowledge, novel. 
\VK{The rationale behind it is that for $\alpha=1$ \ref{eq:RC} is equivalent to $\range(\Koop\TS)\subseteq\range(\TS)$, in which case there exists a bounded $\im$-a.e. Koopman operator $\HKoop\colon\RKHS\to\RKHS$ ~\citep{Kostic2022}. On the other hand, as $\rpar\to0$ \ref{eq:RC} becomes closer to $\range(\Koop\TS)\subseteq\cl(\range(\TS))$ which is always satisfied for universal kernels since $\cl(\range(\TS))=\Lii$~\citep[Chapter 4]{Steinwart2008}. Importantly, as the next example shows, \ref{eq:RC} is weaker condition than the usual regularity conditions; see Appendix~\ref{app:assumptions} for a detailed discussion.}
\begin{example}
\label{ex:rank1_koop}
\VK{Let $X$ be an $\X$-valued random variable with law $\im$. Consider the Markov chain $(X_t)_{t\in\N}$ such that $X_t=X$ for all $t\in\N$. Then $\im$ is an invariant measure and $\Koop=I_{\Lii}$ is the identity map on $\Lii$. Clearly, \ref{eq:RC} holds for all $\alpha\in(0,1]$. On the other hand, since  $\Koop\TS=\TS \HKoop$ for bounded operator $\HKoop=I_{\RKHS}\not\in\HS{\RKHS}$, HS-norm learning rates derived in \cite{Li2022} do not apply.}
\end{example}
In order to study the error of any empirical finite rank estimator $\EEstim$ %
we rely on the 
error decomposition
\begin{equation}
    \error(\EEstim)\leq \underbrace{\norm{\Koop\TS - \TS\RKoop}}_{\text{regularization bias}} + \underbrace{\norm{\TS (\RKoop - \Estim)}}_{\text{rank reduction bias}}  + \underbrace{\norm{\TS(\Estim - \EEstim)}}_{\text{estimator's variance}},
    \label{eq:dec}
\end{equation}
where $\RKoop:=\Creg^{-1}\Cxy$ is the minimizer of the full (i.e. without rank constraint), Tikhonov regularized, HS norm error, and  $\Estim$ is the population version of the empirical estimator $\EEstim$.  

While the last two terms in the r.h.s. of \eqref{eq:dec} depend of the estimator of choice, the first term depends only on the choice of $\RKHS$ and the regularity of $\Koop$ w.r.t. $\RKHS$.  {In this work we focus on the classical kernel-based learning of the Koopman operator~ \cite{LGBPP12,Li2022,Kostic2022}, where one chooses a universal kernel~\citep[Chapter 4]{Steinwart2008} for which $\range(\Koop\TS) \subseteq \cl(\range(\TS))$, and controls the regularization bias with a regularity condition. For details see Rem.~\ref{rem:app_bound} of Appendix~\ref{app:bias}.}

The second source of bias and the estimator's variance in our error decomposition depends on the choice of the low rank estimator. \VK{While throughout this section we consider \ref{eq:RC} for $\rpar\in[1,2]$, we discuss extensions of our results to $\rpar<1$ in Appendix~\ref{app:missspec_optimal}.}  %

\begin{restatable}{theorem}{thmError}\label{thm:error_bound} 
Assume the operator $\Koop$ satisfies
$\sigma_r(\Koop\TS)>\sigma_{r+1}(\Koop\TS)\geq0$ for some $r\in\N$. Let \ref{eq:SD} and \ref{eq:RC} hold for some $\spar\in(0,1]$ and $\rpar\in[1,2]$, respectively, and let $\cl(\range(\TS))=\Lii$. Let
\begin{equation}\label{eq:opt_reg}
    \reg\asymp n^{-\frac{1}{\rpar+\spar}}\,\text{ and }\,\rate^\star_n:= n^{-\frac{\rpar}{2(\rpar+\spar)}}.
\end{equation}
Let $\delta\in(0,1)$. Then, there exists a constant $c\,{>}\,0$, depending only on $\RKHS$, such that for large enough $n\geq r$, with probability at least $1\,{-}\,\delta$ in the i.i.d. draw of 
$\Data$ from $\rho$
\begin{subnumcases}{\error(\EEstim)\leq}
   \sigma_{r+1}(\Koop\TS){+}c\,\rate_n^\star\,\ln \delta^{-1} & {\rm if} \,$\EEstim=\ERRR$, \vspace{.25truecm} \label{eq:error_bound_rrr}
   \\
   \sigma_{r+1}(\TS) +  c\,\rate_n^\star\,\ln \delta^{-1} & {\rm if}\, $\EEstim = \EPCR \,\,{\rm and\,\,} \sigma_r(\TS)>\sigma_{r+1}(\TS)$. \label{eq:error_bound_pcr}
\end{subnumcases}
\end{restatable}
\begin{proof}[Proof Sketch]
{The regularization bias is bounded by $\rcon\, \reg^{\frac{\rpar}{2}}$ by Proposition~\ref{prop:app_bound} of Appendix~\ref{app:bias}.} For the RRR estimator, the {\em rank reduction bias} is upper bounded by $\sigma_{r+1}(\Koop\TS)$, while for PCR by $\sigma_{r+1}(\TS)$. 
The bounds on the variance terms critically rely on the well-known perturbation result for spectral projectors reported in Proposition~\ref{prop:spec_proj_bound}, Appendix~\ref{app:background}. This result is then chained to two versions of the Bernstein inequality in separable Hilbert spaces. The first one is Pinelis-Sakhanenko's inequality and the second is Minsker's inequality extended to self-adjoint HS-operators, Props.~\ref{prop:con_ineq_ps} and ~\ref{prop:con_ineq} in Appendix~\ref{app:concentration_ineq}, respectively. These inequalities provide high probability bounds for the norms of $\Creg^{-1/2}(\ECx\,{-}\,\Cx)$ and ${\Creg^{-1/2}(\ECx\,{-}\,\Cx)\Creg^{-1/2}}$, as well as  $\Creg^{-1/2}(\ECxy\,{-}\,\Cxy)$ and ${\Creg^{-1/2}(\ECxy\,{-}\,\Cxy)\Creg^{-1/2}}$. Combining the bias due to regularization and variance terms, for both estimators we obtain the balancing equation $\reg^{\frac{\rpar}{2}} \,{=}\, \reg^{-\frac{\spar}{2}} \,n^{-\frac{1}{2}}\,\ln\delta^{-1}$, which yields the optimal choice of $\reg$ and the rates. 
\end{proof}
We stress that the number of samples in the previous theorem depends on the problem's complexity, expressed in the constants 
\(
c_{\rm RRR}\,{=}\,\tfrac{1}{\sigma_r^2(\Koop\TS){-}\sigma_{r+1}^2(\Koop\TS)},\,\text{ and }\, c_{\rm PCR}=\tfrac{1}{\sigma_r(\TS){-}\sigma_{r+1}(\TS)}.
\)
Namely, the better the separation of singular values, the smaller number of needed samples. Furthermore, analyzing the bounds \eqref{eq:error_bound_rrr} and \eqref{eq:error_bound_pcr}, we see that faster spectral decay is, in general, preferable. For example, for the Gaussian kernel $\spar$ can be chosen arbitrarily small, yielding the rate $n^{-1/2}$. On the other hand, kernels with slow spectral decay for which $\spar=1$ can give slower rates between $n^{-1/4}$ and $n^{-1/3}$. \VK{Finally, from the variance bounds for RRR and PCR, c.f.~Appendix~\ref{app:variance}, one can specify constants. Namely, in the slower regime when $\rate_n^\star>n^{-1/2}$ we have that $c=a+7.2 \log(10)\sqrt{2 c_{\mathcal{H}}} (1+ a c_{\mathcal{H}}^{(\alpha-1)/2})(\sqrt{\bcon} \wedge \frac{b^{\beta/2}}{\sqrt{1-\beta}})$, while in fastest regime $\rate_n^\star=n^{-1/2}$, there is a significant difference between RRR and PCR since $c$ should be multiplied with the constants $c_{\rm RRR}$ and $c_{\rm PCR}$, respectively.}

As argued in Section~\ref{sec:main_short}, the bound \eqref{eq:error_bound_rrr} indicates that for rank $r$ Koopman operators the error converges to zero w.r.t. the number of training samples, while the bias of PCR is strictly positive. Hence, in order to ensure small error for PCR, high values of the rank parameter might be necessary. To theoretically explain this effect, in Theorem \ref{thm:error_bound_conc} of Appendix~\ref{app:error_bound} we give also lower bounds of operator norm error for the RRR and PCR estimators showing that $\error(\ERRR)$ always concentrates around $\sigma_{r+1}(\Koop\TS)$, while the concentration of $\error(\EPCR)$ around $\sigma_{r
+1}(\TS)$ depends on the \textit{irreducible risk} of the learning problem. To illustrate the tightness of the error concentration bounds we present Example \ref{ex:OU} (see also Appendix~\ref{app:error_bound}). 
\begin{example}\label{ex:OU}
Let $\X \,{=}\, \R$. Consider the 1D equidistant sampling of the Ornstein–Uhlenbeck process, obtained by integrating the Langevin equation of Example~\ref{ex:langevin} with $\beta = 1$ and $U(x) = x^{2}/2$, given by
\(X_{t} \,{=}\, e^{-1}X_{t{-}1} \,{+}\, \sqrt{1{-}e^{-2}}\,\epsilon_t,\), 
where $\{\epsilon_t\}_{t\geq 1}$ are i.i.d. standard Gaussians. For this process it is well-known~\cite{Pavliotis2014} that $\im$ is $\mathcal{N}(0,1)$ and that $\Koop$ admits a spectral decomposition $(\keval_i, \kefun_i)_{i \in \N}$ in terms of Hermite polynomials. We study
the 
family of kernel functions
 $k_{\Pi, \nu}(x, x^{\prime}) := \sum_{i \in \N} \keval_{\Pi(i)}^{2\nu}\kefun_{i}(x)\kefun_i(x^{\prime})$,
where $\Pi$ is a permutation of the indices of the eigenvalues and $\nu$ is a scaling factor. The rationale behind this class of kernels is that by varying $\Pi$ and $\nu$ one morphs the original metric structure of $\Koop$ in a way which is harder and harder to revert when learning from finite sets of data. In particular, for any target rank $r$, setting $\nu:=1/r^2$ and $\Pi$ to the permutation such that $i \mapsto 2r \,{-}\, i \,{+}\, 1$ ($i \,{\leq}\, r$), $i \mapsto i \,{-}\, r$ ($r \,{+}\, 1 \leq i \,{\leq}\, 2r$) and $i \mapsto i$ elsewhere, elementary 
algebra and our concentration bounds 
give
\[
\abs{\error(\EPCR) - e^{- 1/r}} \lesssim n^{-1/2}\ln\delta^{-1},\quad \quad \abs{\error(\ERRR) - e^{-r}}\lesssim n^{-1/2}\ln\delta^{-1}.
\]
We refer the reader to Figure~\ref{fig:good_bad_ugly} and to Section~\ref{sec:exp} for a numerical implementation of this example. 
\end{example}

\VK{We conclude this section with remarks on the tightness of our statistical analysis of operator norm error. Since discussed results are not the main focus of the paper, we present them in Appendix~\ref{app:missspec_optimal}.
\begin{remark}[Lower bound]\label{rem:lowerbound}
The rate $\rate_n^\star = n^{-\frac{\rpar}{2(\rpar+\spar)}}$ guaranteed by \eqref{eq:error_bound_rrr} matches the minimax lower bound for the operator norm error when learning finite rank $\Koop$. Formal statement and its proof is given in Theorem~\ref{thm:min_lower_bound} of Appendix~\ref{app:missspec_optimal}.
\end{remark}
\begin{remark}[Extension to misspecified setting]\label{rem:misspecified}
The optimal rates for HS-norm error of the KRR estimator are developed in \cite{Li2022} under a stronger condition than \ref{eq:RC}. In Theorem~\ref{thm:error_bound_misspec} of Appendix~\ref{sec:extensionSRC} we extended this analysis to PCR and RRR estimators, deriving the optimal operator norm rates that also cover cases when Koopman operator cannot be properly defined as bounded operator on the chosen RKHS space $\RKHS$.
\end{remark}
}
 
\section{Spectral Learning Rates}\label{sec:spectral_rates}
Collecting all the previous results, we are now ready to present our spectral learning rates for the two estimators in a general form. For brevity, we focus on two different type of bounds in which 
(i) we analyse the uniform bound for the whole estimated spectra, and
(ii) we express the estimators' bias in \textit{empirical form} to provide an insight into spuriousness of eigenvalues.
Moreover, we present only eigenvalue estimation bounds, noting that the eigenfunction estimation bounds readily follow from \eqref{eq:eigfunc_approximation_error}. The complete results are presented in detail in Appendix~\ref{app:spec_learning}. 

\begin{restatable}{theorem}{ThmSpecBoundsUniform}\label{thm:spectral_uniform_main}
Let $\Koop$ be a compact self-adjoint operator. Under the assumptions of Theorem \ref{thm:error_bound}, there exists a constant $c>0$, depending only on $\RKHS$, such that for every $\delta\in(0,1)$, for every large enough $n\geq r$ and every $i\in[r]$ with probability at least $1-\delta$ in the i.i.d. draw of $\Data$ from $\rho$
\begin{equation}\label{eq:unif_bias}
\abs{\eeval_i-\keval_{{j(i)}}} \leq \begin{cases}
    \frac{2\sigma_{r+1}(\Koop\TS) }{\sigma_{r}(\Koop\TS)} + c\,\rate_n^\star\ln \delta^{-1}& {\rm if}\,\,\,\,\EEstim=\ERRR, \vspace{.25truecm}\\ 
    \frac{2\sigma_{r+1}(\TS) }{[\sigma_{r}(\Koop\TS)-\sigma_{r+1}^\alpha(\TS)]_+} + c\,\rate_n^\star\ln \delta^{-1} & {\rm if}\,\,\,\,\EEstim = \EPCR.
\end{cases}
\end{equation}
\end{restatable}
The uniform eigenvalue learning rates for RRR and PCR estimators, differ in the estimator's bias. %
While the PCR bias has a factor $\sigma_{r+1}(\TS)$ in the numerator, RRR has $\sigma_{r+1}(\Koop\TS)\leq \sigma_{r+1}(\TS)$. %
The striking difference happens when the Koopman operator is of finite rank. Then, assuming that $r$ is properly chosen, RRR estimator has no bias, and $c_{\rm RRR}$ is typically moderate. On the other hand, 
PCR's bias can be potentially large, depending of the choice of the kernel, and choosing higher rank increases $c_{\rm PCR}$, thus requiring larger sample sizes.  
Therefore, even in well-conditioned  problems (self-adjoint operator) the \textit{spurious eigenvalues may arise purely from the learning method}. To facilitate detection of such occurrences, we further provide an empirical estimator of the bias of both methods and illustrate their use experimentally in Section~\ref{sec:exp}. 

\begin{restatable}{theorem}{ThmSpecBounds}\label{thm:spectral_main}
Under the assumptions of Theorems \ref{thm:error_bound} and \ref{thm:spectral_uniform_main},
there exists a constant $c>0$, depending only on $\RKHS$, such that for large enough $n\geq r$ and every $i\in[r]$ with probability at least $1-\delta$ in the i.i.d. draw of $\Data$ from $\rho$
\begin{equation}\label{eq:spectral_bounds_evals}
\abs{\eeval_i-\keval_{{j(i)}}} \leq 
\begin{cases}    \emetdist_i\,\sigma_{r+1}(\ECx^{-1/2}\ECxy) + c\,\rate^\star_n\,\ln \delta^{-1}, & \EEstim=\ERRR,\vspace{.25truecm}\\    \emetdist_i\,\sqrt{\sigma_{r+1}(\ECx)} + c\,\rate^\star_n\,\ln \delta^{-1},& \EEstim = \EPCR.\end{cases} 
\end{equation}
\end{restatable}

We remark that when $\Koop$ is of finite rank $r$, the bound above for the RRR estimator reduces to
\[
\abs{\eeval_i-\keval_i} \leq c\,\rate^\star_n\,\ln \delta^{-1}\;\text{ and }\;
\norm{\ekefun_{i} - \kefun_{i}}^2\leq \frac{2\,c\,\rate^\star_n\,\ln \delta^{-1} %
}{[\gap_i(\ERRR) - 3\,c\, \rate^\star_n\,\ln\delta^{-1} %
]_{+}},
\]
see Cor.~\ref{cor:finite_rank_koop_rrr} in Appendix~\ref{app:spec_learning}. Hence, in this case RRR algorithm can learn all the eigenvalues and eigenfunctions of $\Koop$ with rate $\rate_n^\star=n^{-\frac{\rpar}{2(\rpar+\spar)}}$. On the other hand, even in this case, the bounds for the PCR estimator do not guarantee unbiased estimation of Koopman eigenvalues and eigenfunctions.

\textbf{Choosing {$\reg$ and $r$}.} The bias term $\sigma_{r+1}(\Koop\TS) / \sigma_{r}(\Koop\TS)$ appearing in \eqref{eq:unif_bias} represents the theoretical limit when estimating eigenvalues using RRR. It reflects the capacity of the RKHS to detect the separation of the leading $r$ Koopman eigenvalues from the rest of its spectra. If $\Koop$ has infinite rank and slowly decaying eigenvalues, estimating the leading ones becomes challenging, since increasing $r$ leads to smaller operator norm error, but larger bias. Luckily, in many practical problems there is a separation of time-scales in the dynamics and the above ratio can be controlled by choosing $r$ appropriately. While we do not have access to $\Koop\TS$, we can still choose $r$ via the empirical operator $\ECx^{-1/2}\ECxy$, see Proposition~\ref{prop:svals_rate} of Appendix~\ref{app:error_bound}. Note also that the optimal $\reg$ depends on $\rpar$ which is typically unknown. In practice, one can implement a standard grid-search CV procedure for time series to tune this parameter.

\textbf{Spectral Bias as a Tool for Model Selection.} In equation \eqref{eq:spectral_bounds_evals}, the data dependent quantities $\widehat{s}_i(\ERRR):=\emetdist_i\,\sigma_{r+1}(\ECx^{-1/2}\ECxy)$ and $\widehat{s}_i(\EPCR):=\emetdist_i\,\sigma_{r+1}(\ECx)$  represent the \emph{empirical spectral biases} of RRR and PCR estimators of the Koopman operator, respectively. When they are small enough, the spectral estimation error is dominated by the same variance term, which decreases as the number of samples grows. Therefore, given a number of different kernels, we propose to select the best one (w.r.t. spectral estimation) by choosing the smallest spectral bias. This is illustrated in the Alenine Dipeptide example of following section.

\VK{\textbf{Normal operators.} Since Davis-Kahan theorem~\cite{DK1970} also holds for normal operators, the results in this section apply whenever $\Koop\adjKoop = \adjKoop\Koop$. While in this case Koopman eigenfunctions remain orthogonal in $\Lii$, the eigenvalues are in general complex. On the other hand, extension beyond normal compact operators asks for involved spectral perturbation analysis and a new statistical learning theory.}

\section{Experiments}\label{sec:exp} 
We illustrate various aspects of our theory with simple experiments. They have been implemented in Python using the library Kooplearn (available at \url{https://github.com/CSML-IIT-UCL/kooplearn}) to fit the PCR and RRR estimators. Full details are in Appendix~\ref{app:exp}.

\textbf{Learning the Spectrum of the Ornstein–Uhlenbeck Process.} In this experiment we designed three different kernel functions (the ``{\em good}\hspace{.05truecm}'’, the ``{\em bad}\hspace{.05truecm}'' and the ``{\em ugly}\hspace{.05truecm}'') to illustrate how an unseemly kernel choice can induce catastrophic biases in the estimation of Koopman eigenvalues. 
We focus on the uniformly sampled Ornstein-Uhlenbeck (OU) process, discussed in Example~\ref{ex:OU}, relying on the spectral decomposition of its Koopman operator $(\keval_{i}, \kefun_{i})_{i \in \N}$ to design the three kernel functions. 
The {\em good} kernel is just the sum of the leading $T = 53$ terms of the spectral decomposition of $\Koop$, i.e. $k_{{\rm good}}(x,y) := \sum_{i=1}^{T}\keval_{i}\kefun_{i}(x)\kefun_{i}(y)$. The associated RKHS coincides with the leading eigenspace of $\Koop$, and {\em no deformation of the metric structure} takes place, so that the injection map $\TS \colon \RKHS \hookrightarrow \Lii$ is a partial isometry. The {\em bad} kernel is defined according to the construction presented in Example~\ref{ex:OU} for $\nu = 1/r^{2}$ where $r$ is the rank of the estimator. For this kernel,  the introduced bias is innocuous for RRR, but lethal for PCR.  Finally, the {\em ugly} kernel corresponds to $\nu = r^{2}$, introducing large quotients $\sigma_{r+1}(\Koop\TS)/\sigma_{r}(\Koop\TS)$ and $\sigma_{r+1}(\TS)/\sigma_{r}(\TS)$, and, hence, an irreparable bias in both estimators. 

Figure~\ref{fig:good_bad_ugly} depicts the distribution of the eigenvalues estimated by PCR and RRR over 50 independent simulations, against the ground truth. For both algorithms each simulation is comprised of $20000$ training points, the regularization is $\reg = 10^{-4}$ and the rank is $r = 3$. The three largest eigenvalues of $\Koop$ are correctly estimated by both algorithms for $k_{{\rm good}}$ and by RRR for $k_{{\rm bad}}$. On the contrary, the distribution of the eigenvalues for $k_{{\rm ugly}}$ (and $k_{{\rm bad}}$ for PCR) does not concentrate around any true eigenvalue of $\Koop$, signaling the presence of {\em spurious eigenvalues} in the estimation. 

\begin{figure}[t!]
    \centering
\includegraphics[width=\textwidth]{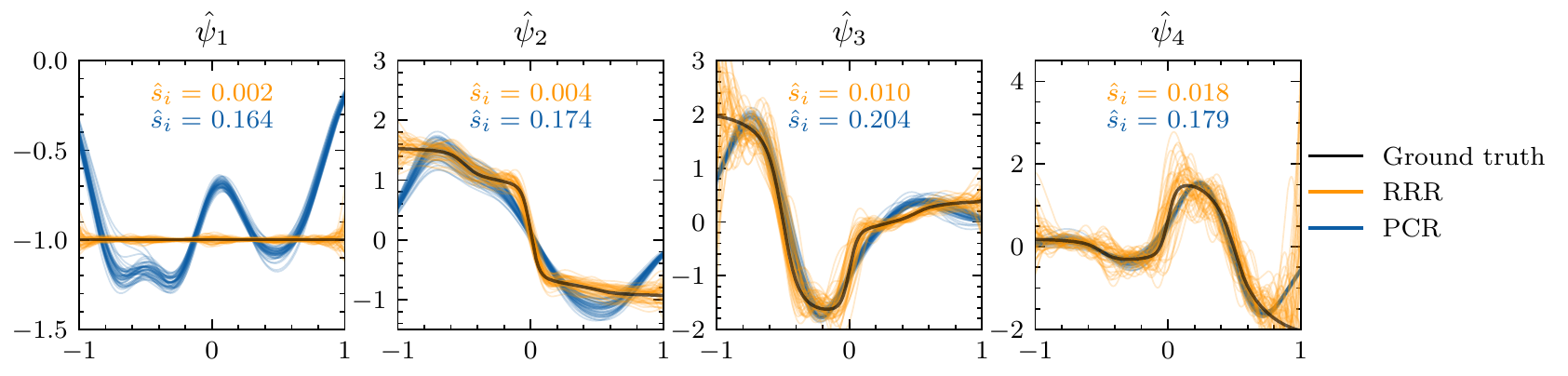}
\vspace{-.4truecm}
\caption{Estimated eigenfunctions $\erefun_i$ of a Langevin dynamics vs. ground truth. The average empirical biases $\widehat{s}_i$, $i\in[4]$ are discussed at the end of Section \ref{sec:spectral_rates}. The results correspond to 50 independent estimations on 2000 training points each. PCR and RRR estimators were fitted with the same parameters: Gaussian kernel of length scale $0.175$, $\reg = 10^{-5}$ and $r=4$. }
\label{fig:langevin_eigfun_approximation}
\vspace{-.3truecm}
\end{figure}
\begin{figure}[th!]
   \centering
\includegraphics[width=0.5\columnwidth]{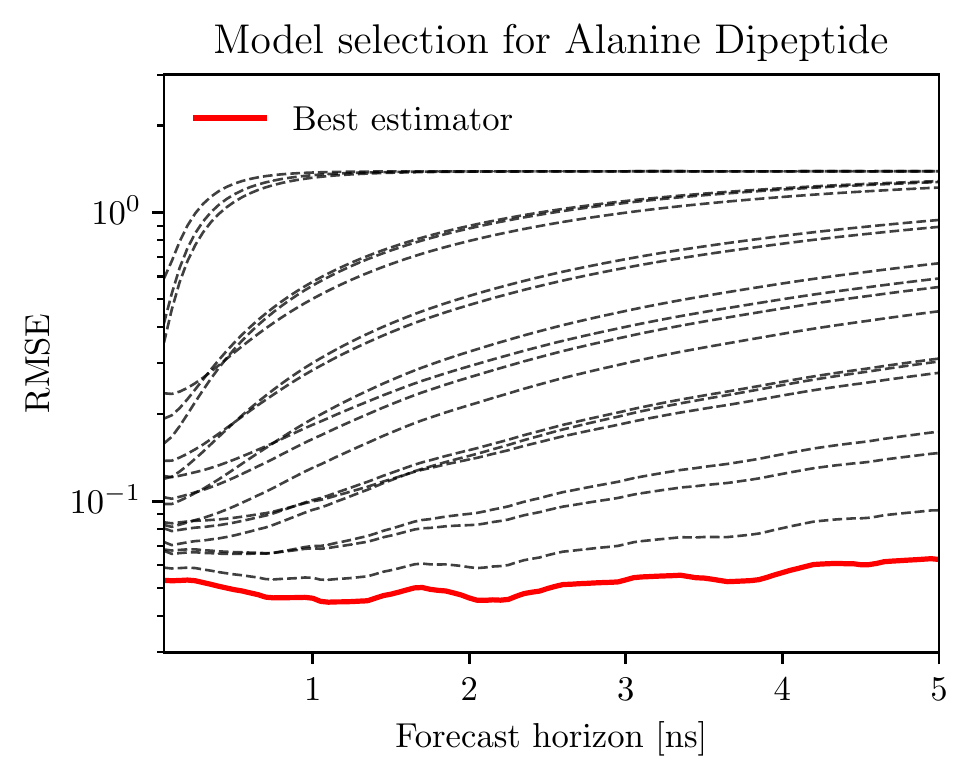}
\vspace{-.35truecm}
{\caption{Forecasting RMSE on the Alanine Dipeptide dataset for 19 different RRR estimators, each corresponding to a different kernel, which show how the best model, according to the empirical spectral bias metric, also attains the best forecasting performances by a large margin.}
\label{fig:koopman_model_selection}}
\end{figure}

\textbf{A Realistic Example: Langevin Dynamics.} %
Because of its ubiquitous use in modelling real systems, we now study a numerical implementation of the Langevin dynamics Example~\ref{ex:langevin} with $\beta = 1$ and a potential $U(x)=4(x^8{+}0.8e^{-80x^2}{+}0.2e^{-80(x{-}0.5)^2} + 0.5e^{-40(x{+}0.5)^2})$ that is a mixture of three Gaussians barriers at $x \in \{ - 0.5, 0, 0.5\}$ and a smooth ``bounding'' 
term $\propto x^{8}$ constraining most of the equilibrium distribution in the interval $[-1, 1]$, see~\cite{Schwantes2015}. %
In Figure~\ref{fig:langevin_eigfun_approximation}, for $i\in[r]$, we compare the $\ekefun_i$ estimated by PCR and RRR %
against the ground truth $\kefun_i$. The visible difficulty of PCR compared to RRR in estimating eigenfunctions is nicely explained by larger values of the empirical bias for PCR, which we report in the upper part of the figure. The reference eigenpairs of $\Koop$ have been obtained by diagonalizing a finely discretized approximation of the infinitesimal generator 
(see Appendix~\ref{app:background}).

%

%
%
%
%
%
%

\textbf{Spectral Bias and Model Selection: the Case of Alanine Dipeptide.} 
In this example we show that minimizing the first term on the r.h.s. of~\eqref{eq:spectral_bounds_evals} over a validation dataset, is also a good criterion for Koopman model selection. We use a realistic simulation of the small molecule Alanine Dipeptide already discussed in~\cite{Wehmeyer2018, Kostic2022}. We trained 19 RRR estimators each corresponding to a different kernel and then we evaluated the forecasting RMSE on 2000 initial conditions drawn from a test dataset. In Figure~\ref{fig:koopman_model_selection} we report these errors, highlighting the model with the smallest average empirical spectral bias~\eqref{eq:spectral_bounds_evals} evaluated on 5000 validation points.

\section{Conclusion}\label{sec:concl}
We established minimax optimal rates for the operator norm error in the Koopman regression problem, which we then used to derive sharp estimation bounds for eigenvalues and eigenfunctions of the Koopman operator associated with a time-invariant Markov chain. We considered two important estimators that implement either principal component regression (PCR) or reduced rank regression (RRR) to learn a linear operator on a reproducing kernel Hilbert space.
Our bounds indicate that RRR may be advantageous over PCR (also known as EDMD, the de-facto estimator in the data-driven dynamical system literature) which may exhibit a larger estimation bias. This ultimately depends on the choice of the kernel, which significantly impacts the rate. A bad choice of the kernel could also introduce spurious eigenvalues, a phenomena which has been observed in the literature and which is now explained by our theory.
Finally, we proposed a method to detect spuriousness in practice, which can be used also as a kernel selection tool. A limitation of this work is that it applies to \VK{compact normal}  operators only. While many real dynamical systems  involve such operators, 
in the future our analysis may be extended using more sophisticated spectral perturbation theory.

\VK{\paragraph{Acknowledgements.} This work was supported in part from the PNRR MUR Project PE000013 CUP J53C22003010006 “Future Artificial Intelligence Research (FAIR)“, funded by the European Union – NextGenerationEU, and EU Project ELIAS under grant agreement No. 101120237.}

\bibliographystyle{apalike}
{
\bibliography{bibliography}
}
\newpage
\appendix

\begin{center}
{\Large \bf Supplementary Material} 
\end{center}

The supplementary material is organized as follows.
\begin{itemize}
    \item  Appendix~\ref{app:background} contains additional background on stochastic processes with self-adjoint Koopman operator, on Markov processes and spectral theory. Additionally, it contains a notation table.  
    \item  Appendix~\ref{app:kooplearn} discusses learning of the Koopman operator with kernel-based methods. 
      
    \item Appendix~\ref{app:approach} contains details for the content presented in Section \ref{sec:approach}, notably proving the key perturbation result of Theorem~\ref{thm:spectral_perturbation}.

\item Appendix~\ref{app:error} contains details of the content presented in Section \ref{sec:error}. In particular,
in Appendix~\ref{app:assumptions} we discuss the main assumptions and their relationship with the existing literature, 
in Appendix~\ref{app:bias} we prove the bounds of different bias terms: bias due to RKHS, bias due to estimator of choice and bias in the effective rank estimation,
in Appendix~\ref{app:variance} we prove the bounds of the corresponding variance terms, and
in Appendix~\ref{app:error_bound} we show the error bounds for the RRR and PCR estimators %
under stronger \ref{eq:RC} condition, while in Appendix~\ref{sec:extensionSRC} we extend them and prove the matching lower bound in Theorem~\ref{thm:min_lower_bound} in Appendix~\ref{app:missspec_optimal}.
\item In Appendix~\ref{app:spec_learning} we prove the spectral learning rates of Section~\ref{sec:spectral_rates} in more detailed form.
\item Finally, in Appendix~\ref{app:exp} we provide more details on the experimental section, as well as present additional experiments.

\end{itemize}

\renewcommand{\arraystretch}{1.26}
\begin{table}[!h]\label{tab:notation}
\centering
\makebox[\textwidth]{\begin{tabular}{c|c||c|c}
\toprule
notation & meaning & notation & meaning  \\ 
\midrule 
$\wedge$ & minimum & $\vee$ & maximum \\ \hline
$[\,\cdot\,]$ & set $\{1,2\ldots,\cdot\}$ & $[\,\cdot\,]_+$ & nonegative part of a number \\\hline
$\X$ & state space of the Markov chain & $(X_t)_{t\in\N}$ & time-homogeneous Markov chain \\\hline
$\transitionkernel$ & transition kernel of the Markov chain & $\im$ & invariant measure of the Markov chain  \\\hline
$\Lii$ & L2 space of functions on $\X$ w.r.t. measure $\im$ & $\Koop$ & Koopman operator on $\Lii$ \\\hline
$\Ker(\cdot)$  & null space of an operator & $\range(\cdot)$ &  range of an operator \\\hline
$\cl(\cdot)$ & closure of a subspace  & $\tr(\cdot)$ & trace of an operator \\\hline
$\sigma_i(\cdot)$ & $i$-th singular value of an operator & $\eval_i(\cdot)$ & $i$-th eigenvalue of an operator \\\hline
$\SVDr{\,\cdot\,}$ & $r$-truncated SVD of an operator & $\Id$ & identity operator \\\hline
$k(x,y)$ & kernel & $\phi$  & canonical feature map \\\hline
$\RKHS$ &reproducing kernel Hilbert space & $\TS$ & canonical injection $\RKHS \hookrightarrow\Lii$ \\\hline
$\HS{\spH,\spG}$ & space of Hilbert-Schmidt operators $\spH\to\spG$ & $\HSr$  & set of rank-$r$ Hilbert-Schmidt operators on $\RKHS$\\\hline
$\norm{A}$ & operator norm of an operator $A$ & $\hnorm{A}$ & Hilbert-Schmidt norm of operator $A$ \\\hline
$\TZ$ & restriction of the Koopman operator to $\RKHS$ & $\CME$ & conditional mean embedding \\\hline
$\sigma_j$ & $j$-th singular value of $\TS$ & $J$ & countable index set of singular values of $\TS$ \\\hline
$\ell_j$ & $j$-th left singular function of $\TS$ &  $h_j$ & $j$-th right singular function of $\TS$ \\ \hline
$\one$ &  function in $\Lii$ with the constant output 1 & $\reg$ & regularization parameter \\\hline
$\Risk$ & true risk & $\error$ & operator norm error \\\hline
$\ExRisk$ & excess risk, i.e. HS norm error & $\IrRisk$ & irreducible risk \\\hline
$\Data$ & dataset $(x_i,y_i)_{i\in[n]}$ & $\ERisk$ & empirical risk\\\hline
$\ES$ & sampling operator of the inputs & $\EZ$ & sampling operator of the outputs \\\hline
$\Estim$ & population Koopman estimator in $\HS{\RKHS}$ & $\EEstim$ & empirical Koopman estimator in $\HS{\RKHS}$  \\\hline
$\RKoop$ & population KRR estimator & $\ERKoop$ & empirical KRR estiamator \\\hline
$\PCR$ & population PCR estimator & $\EPCR$ & empirical PCR estiamator \\\hline
$\RRR$ & population RRR estimator & $\ERRR$ & empirical RRR estiamator \\\hline

$\Cx$ & covariance operator &$\ECx$ & empirical covariance operator \\\hline
$\Creg$ & regularized covariance operator &$\ECreg$ & regularized empirical covariance operator \\\hline
$\Cxy$ & cross-covariance operator &$\ECxy$ & empirical cross-covariance operator \\\hline
$\Kx$ & input kernel matrix &$\Ky$ &  output kernel Gramm matrix \\\hline
$\Kreg$ & regularized input kernel matrix &
$\Kyx$ & cross-kernel matrix \\\hline

$\TB$ & operator $\Creg^{-1/2}\Cxy$ &
$\EB$ & empirical operator $\ECreg^{-1/2}\ECxy$\\\hline
$\TP$ & spectral projector & $\EP$ & empirical spectral projector \\\hline

$\metdist$ & metric distortion & $\emetdist$ & empirical  metric distortion \\\hline
$\keval$ & Koopman eigenvalue  & $\eeval$ & eigenvalue of the empirical estimator \\\hline
$\kefun$ & Koopman eigenfunction in $\Lii$ & $\ekefun$ & empirical eigenfunction in $\Lii$ \\\hline
$\erefun$ & right empirical eigenfunction  &  $\elefun$ & left empirical eigenfunction \\\hline
$\gap_J(\cdot)$ & spectral gap of an operator w.r.t. indices $J$ & $\cond$ & eigenvalue condition number \\ \hline
$\bcon$ & boundness constant & $P_\RKHS$ & orthogonal projector in $\Lii$  onto $\range(\TS)$\\ \hline
$\rpar$ &regularity parameter & $\rcon$  & regularity constant \\\hline
$\spar$ &spectral decay parameter& $\scon$ & spectral decay  constant\\\hline
$\epar$ &embedding parameter & $\econ$ & embedding constant\\ \hline
\bottomrule
\end{tabular}}
\caption{Summary of used notations.}
\end{table}
\renewcommand{\arraystretch}{1}

\section{Notation and Background}\label{app:background}

First, we briefly recall the basic notions related to  Markov chains and  Koopman operators and refer to~\cite{Lasota1994, Meyn1993, Mauroy2020}  for further details.

Let  $\mathbf{X} := \left\{X_{t} \colon t\in \N \right\}$ be a family of random variables with values in a measurable space $(\X, \sigalg)$, called state space. We call $\mathbf{X}$  a  {\em Markov chain} if $\PP\{ X_{t+1} \in B \,\vert\, X_{[t]} \} = \PP\{X_{t + 1} \in B \,\vert\, X_t \}$. 
Further, we call $\mathbf{X}$ {\em time-homogeneous} if there exists   $\transitionkernel\colon \X \times \sigalg \to [0,1]$, called {\it transition kernel}, such that,  for every $(x, B) \in \X \times \sigalg$ and every $t \in \N$,
\[
\mathbb{P}\left\{X_{t + 1} \in B \middle| X_{t} = x \right\} = \transitionkernel (x,B).
\]

A large class of Markov chains consists of these who posses an \textit{invariant measure} $\im$ satisfying $\pi(B) {=} \int_{\X} \pi(dx)p(x,B)$, $B\in\sigalg$, see e.g.~\cite{prato1996}. For those, we can consider the space of square integrable functions on $\X$ relative to the measure $\pi$, denoted as $\Lii$, and define \textit{Markov transfer operator}, i.e. Stochastic \textit{Koopman operator}, $\Koop \colon \Lii\to\Lii$ 
\begin{equation}\label{eq:Koopman_app}
	\Koop f(x) := \int_{\X} p(x, dy)f(y) = \mathbb{E}\left[f(X_{t + 1}) \middle | X_{t} = x\right], \quad f\in\Lii,\,x\in\X.
\end{equation}
Since it easy to see that $\norm{\Koop} = 1$, we conclude that  the Markov transfer operator is a bounded linear operator.

This work focuses on the Markov chains that originate from a dynamical system that is time-reversal invariant, which as a consequence has that Koopman operator on the $\Lii$ space is self-adjoint.  Since a many microscopical equations of motion in both classical and quantum physics are time-reversal invariant, learning self-adjoint Koopman operators is of
paramount importance in the field of machine learning for physical sciences.

Next, we discuss the (overdamped) Langevin equation
\begin{equation*}
    dX_{t} = -\nabla U(X_t)dt + \sqrt{2\beta^{-1}}dW_t,
\end{equation*}
where $W_{t}$ is a Wiener process.
For any $f \in C^{2}(\R^{d})$ we let $u(x, t) := \EE[f(X_{t})|X_{0} = x]$. As showed in Section 6.3, pp. 95-96 of Ref.~\cite{Varadhan2007},  $u(x, t)$ is the solution of the {\em backward Kolmogorov equation}
\begin{equation}\label{eq:backward_kolmogorov}
    \partial_{t}u = \mathcal{L}u \qquad u(x, 0) = f(x),
\end{equation}
where $(\mathcal{L} f)(x) := \beta^{-1}\nabla^{2}f (x) -\nabla U (x)\cdot \nabla f(x)$. A straightforward calculation shows that $\mathcal{L}$ is a self-adjoint operator with respect to the scalar product $\scalarp{f,g} := \int_{X} f(x)g(x) \im (dx)$ where $\im (dx)$ is the Boltzmann (invariant) distribution $\im (dx) := Z^{-1}e^{-\beta U(x)}dx$ for the process $X_{t}$ ($Z$ being just a normalizing constant). Technically, $\mathcal{L}$ is the infinitesimal generator of the semigroup associated to the Markov process $X_{t}$ on $\Lii$ and completely defines~\cite{Bakry2014} the Koopman operator via the unique solution of~\eqref{eq:backward_kolmogorov}. As usual, the infinitesimal generator is well defined on a dense domain of $\Lii$, that is $C^2(\R^{d})$.

\subsection{Spectral Decomposition}

Recalling that for a bounded a bounded linear operator $A$ on some Hilbert space $\mathcal{H}$ the {\em resolvent set} of the operator $A$ is defined as ${\rm Res}(A) := \left\{\lambda \in \C \colon A - \lambda \Id \text{ is bijective} \right\}$, and its {\it spectrum} $\Spec(A):=\C\setminus\{\Res(A)\}$, let $\lambda\subseteq \Spec(A)$ be isolated part of spectra, i.e. both $\lambda$ and $\mu:=\Spec(A)\setminus\lambda$ are closed in $\Spec(A)$. Than, the \textit{Riesz spectral projector} $P_\lambda\colon\spH\to\spH$ is defined by 
\begin{equation}\label{eq:Riesz_proj}
P_\lambda:= \frac{1}{2\pi}\int_{\Gamma} (z \Id-A)^{-1}dz,
\end{equation}
where $\Gamma$ is any contour in the resolvent set $\Res(A)$ with $\lambda$ in its interior and separating  $\lambda$ from $\mu$. Indeed, we have that $P_\lambda^2 = P_\lambda$ and $\spH = \range(P_\lambda) \oplus \Ker(P_\lambda)$ where $ \range(P_\lambda)$ and $\Ker(P_\lambda)$ are both invariant under $A$ and $\Spec(A_{\vert_{\range(P_\lambda)}})=\lambda$,  
$\Spec(A_{\vert_{\Ker(P_\lambda)}})=\mu$. Moreover, $P_\lambda + P_\mu = \Id$ and $P_\lambda P_\mu = P_\mu P_\lambda = 0$.

Finally if $A$ is {\em compact} operator, then the Riesz-Schauder theorem, see e.g. \cite{Reed1980}, assures that $\Spec(T)$ is a discrete set having no limit points except possibly $\lambda = 0$. Moreover, for any nonzero $\lambda \in \Spec(T)$, then $\lambda$ is an {\em eigenvalue} (i.e. it belongs to the point spectrum) of finite multiplicity, and, hence, we can deduce the spectral decomposition in the form
\begin{equation}\label{eq:Riesz_decomp}
A = \sum_{\lambda\in\Spec(A)} \lambda \, P_\lambda,
\end{equation}
where geometric multiplicity of $\lambda$, $r_\lambda:=\rank(P_\lambda)$, is bounded by the algebraic multiplicity of $\lambda$. If additionally $A$ is normal operator, i.e. $AA^*= A^* A$, then $P_\lambda = P_\lambda^*$ is orthogonal projector for each $\lambda\in\Spec(A)$ and $P_\lambda = \sum_{i=1}^{r_\lambda} \psi_i \otimes \psi_i$, where $\psi_i$ are normalized eigenfunctions of $A$ corresponding to $\lambda$ and $r_\lambda$ is both algebraic and geometric multiplicity of $\lambda$.

We conclude this section with well-known perturbation bounds for eigenfunctions and spectral projectors of self-adjoint compact operators.

\begin{proposition}[\cite{DK1970}]\label{prop:davis_kahan}
Let $A$ be compact self-adjoint operator on a separable Hilbert space $\RKHS$. Given a pair $(\eeval,\ekefun)\in \C \times\RKHS$ such that $\norm{\ekefun}=1$, let $\keval$ be the eigenvalue of $A$ that is closest to $\eeval$ and let $\kefun$ be its normalized eigenfunction. If $\widehat{g}:=\min\{\abs{\eeval-\eval}\,\vert\,\eval\in\Spec(A)\setminus\{\keval\}\}>0$, then $\sin(\sphericalangle(\ekefun,\kefun))\leq \norm{A\ekefun-\eeval\ekefun} / \widehat{g}$.
\end{proposition}
\begin{proposition}[\cite{zwald2005}]\label{prop:spec_proj_bound}
Let $A$ and $\widehat{A}$ be two compact operators on a separable Hilbert space. For nonempty index set $J\subset\N$ let
\[
\gap_J(A):=\min\left\{\abs{\eval_i(A)-\eval_j(A)}\,\vert\, i\in\N\setminus J,\,j\in J \right\}
\]
denote the spectral gap w.r.t $J$ and let $P_J$ and $\widehat{P}_J$ be the corresponding spectral projectors of $A$ and $\widehat{A}$, respectively.  If $A$ is self-adjoint and for some $\norm{A-\widehat{A}} < \gap_J(A) $, then
\[
\norm{P_J - \widehat{P}_J}\leq \frac{\norm{A-\widehat{A}}}{\gap_J(A)}.
\] 
\end{proposition}

\subsection{Koopman Operator and Mode Decomposition} 
The main reason for the use of (stochastic) \textit{Koopman operator} in dynamical systems lies in the fact that its linearity can be exploited to compute a spectral decomposition.  Indeed,  in many situations, and notably for compact Koopman operators, there exist scalars $\keval_i\in\C$, called Koopman eigenvalues, and observables  $\refun_i\in\Lii \setminus\{0\}$, called Koopman eigenfunctions, such that $\Koop\kefun_i \hspace{.05truecm}{=}\hspace{.105truecm} \keval_i\kefun_i$. Then, 
the dynamical system can be decomposed
into superposition of simpler signals that can be used in different tasks such as system identification and control, see e.g. \cite{Brunton2022}.
This becomes particularly elegant when $\Koop$ is compact, then for every observable $f\in\Lii$ there exist corresponding scalars $\gamma_i^f\in\C$ known as Koopman modes of $f$, such that%
\begin{equation}\label{eq:koopman_MD}
\Koop^t f(x) = \EE[f(X_t)\,\vert\, X_0 = x] =  \sum_{j\in\N}\keval_j^t \gamma^f_j \kefun_j(x), \quad x\in\X,\,t\in\N.
\end{equation}
This formula is known as \textit{Koopman Mode Decomposition} (KMD) \cite{Budisic2012,AM2017}. It decomposes the expected dynamics observed by $f$ into \textit{stationary} modes $\gamma_j^f$ that are combined with \textit{temporal changes} governed by eigenvalues $\keval_j$ and \textit{spatial changes}  
governed by the eigenfunctions $\kefun_j$. 

KMD is closely related to general theory of spectral decomposition for bounded linear operators, and in particular the Riesz decomposition theorem. Namely, KMD of a compact self-adjoint Koopman operator can be stated as
\begin{equation}\label{eq:KMD_2}
\Koop^t f(x) = \EE[f(X_t)\,\vert\, X_0 = x] =  \sum_{j\in\N}\keval_j^t \,\scalarp{\kefun_j,f}\,\kefun_j(x), \quad f\in\Lii, x\in\X,\,t\in\N.
\end{equation}

\section{Kernel-Based Learning of the Koopman Operator}\label{app:kooplearn}
In many practical scenarios the transition kernel $p$, hence $\Koop$, is unknown, but data from one or multiple system trajectories are available. In such situations a learning framework called Koopman operator regression was proposed in~\cite{Kostic2022} to estimate Koopman operator on $\Lii$ using reproducing kernel Hilbert spaces (RKHS). More precisely,  let $\RKHS$ be an RKHS with kernel $k:\X \times \X \rightarrow \mathbb{R}$ \cite{aron1950} and let $\phi :\X \to \RKHS$ be an associated feature map, such that $k(x,y) = \scalarp{\phi(x),\phi(y)}$ for all $x,y \in \X$.  %
We assume that $k(x,x) \leq c_{\RKHS} < \infty$, $\pi$- almost surely.  
This ensures that $\RKHS \subseteq \Lii$ and the injection operator $\TS \colon \RKHS \to \Lii$ given by $(\TS f)(x)=f(x)$, $x\in\X$
is a well defined Hilbert-Schmidt operator~\cite{caponnetto2007,Steinwart2008}. 
Then, the Koopman operator restricted to $\RKHS$ is given by 
\[
\TZ := \Koop\TS \colon \RKHS\to\Lii.
\]
Since operator $\TZ$ is, unlike $\Koop$, Hilbert-Schmidt, one can aim to approximate $\TZ$ by minimizing the following {\em risk} $\Risk(\Estim)= \EE_{x\sim \im} \sum_{i \in \N}  \EE \big[ (h_i(X_{t+1}) - (\Estim h_i)(X_t))^2\,\vert X_t  = x\big]$ over Hilbert-Schmidt operators $\Estim\in\HS{\RKHS}$, where $(h_i)_{i\in\N}$ is an orthonormal basis of $\RKHS$. Moreover, there is a bias-variance decomposition of the risk  $\Risk(\Estim)=\IrRisk + \ExRisk(\Estim) $, where 
\begin{equation}
    \label{eq:ex_ir_risk}
\IrRisk=\hnorm{S}^2-\hnorm{Z}^2\geq0\;\text{ and }\; \ExRisk(\Estim)=\hnorm{Z-S\Estim}^2,
\end{equation}
are the irreducible risk (i.e. the variance term in the classical bias-variance decomposition) and the excess risk, respectively. This can be equivalently expressed in the terms of embedded dynamics in RKHS as:
\begin{equation}
    \label{eq:true_risk_cme}
\underbrace{ \EE_{(x,y)\sim \rho} \norm{\phi(y) - \Estim^*\phi(x)}^2}_{\Risk(\Estim)} =  \underbrace{\EE_{(x,y)\sim\rho}\norm{\CME(x) - \phi(y) }^2}_{\IrRisk} +  \underbrace{\EE_{x\sim \im} \norm{\CME(x) - \Estim^*\phi(x)}^2}_{\ExRisk(\Estim)},
\end{equation}
where the regression function $\CME\colon\X\to\RKHS$ is defined as $\CME(x):=\EE[\phi(X_{t+1})\,\vert\,X_t = x]=\int_{\X}p(x,dy)\phi(y)$, $x\in \X$, and is known as the \textit{conditional mean embedding} (CME) of the conditional probability $p$ into $\RKHS$. It was also shown that using universal kernels one can approximate the restriction of Koopman arbitrary well, i.e. excess risk can be made arbitrarily small $\inf_{\Estim\in\HS{\RKHS}} \ExRisk(\Estim)= 0$.

Therefore, to develop estimators one can consider the problem of minimizing the Tikhonov regularized risk   
\begin{equation}\label{eq:KOR_reg}
\min_{\Estim \in \HS{\RKHS}} \Risk^\reg(\Estim){:=}\Risk(\Estim) + \reg\hnorm{\Estim}^2,
\end{equation}
where $\reg>0$. Denoting the covariance matrix as $\Cx := \TS^*\TS = \EE_{x\sim\im} \phi(x)\otimes \phi(x)$ and cross-covariance matrix $\Cxy : = \TS^*\TZ = \EE_{(x,y)\sim\rho} \phi(x)\otimes \phi(y)$, where $\rho(dx,dy):=\im(dx)\transitionkernel(x,dy)$ is the joint probability measure of two consecutive states of the Markov chain, and regularized covariance as $\Creg:=\Cx+\reg\Id_\RKHS$, one easily shows that $\RKoop:=\Creg^{-1} \Cxy$ is the unique solution of \eqref{eq:KOR_reg} which is known as the Kernel Ridge Regression (KRR) estimator of $\Koop$. 

To approximate the leading eigenvalues of the Koopman operator low rank estimators have been also considered. Notably, Principal Component Regression (PCR) estimator given by $\SVDr{\Cx}^\dagger \Cxy$, where $\SVDr{\cdot}$ denotes the $r$-truncated SVD of the Hilbert-Schmidt operator. However, it is observed that both KRR and PCR estimators can fail in estimating well the leading Koopman eigenvalues. To mitigate this, Reduced Rank Regression (RRR) estimator has been introduced in \cite{Kostic2022} as the optimal one that solves \eqref{eq:KOR_reg} with an additional rank constraint by minimizing over the class of rank-$r$ HS operators $\HSr:=\{\Estim \in\HS{\RKHS}\,\vert\,\rank(\Estim)\leq r\}$, where $1\leq r <\infty$, i.e.
\begin{equation}\label{eq:KOR_RRR}
\Creg^{-1/2} \SVDr{ \Creg^{-1/2} \Cxy } = \argmin_{\Estim \in \HSr} \Risk^\reg(\Estim).
\end{equation}

Now, assuming that data $\Data = \{(x_i,y_i)\}_{i\in[n]}$ is collected, the estimators are typically obtained via the regularized empirical risk $\ERisk^\reg(\Estim){:=}\frac{1}{n}\sum_{i\in[n]} \norm{\phi(y_i) - \Estim^*\phi(x_i)}^2 +  \reg\hnorm{\Estim}^2$ minimization (RERM). Introducing  the sampling operators for data $\Data$ and RKHS $\RKHS$ by
\begin{align*}
    \ES \colon \RKHS \to \R^{n} \quad \text{ s.t. }  f \mapsto \tfrac{1}{\sqrt{n}}[ f(x_{i})]_{i \in[n]} & \quad \text{ and } &  \EZ \colon \RKHS \to \R^{n} \quad \text{ s.t. }  f \mapsto \tfrac{1}{\sqrt{n}}[ f(y_{i})]_{i \in[n]},
\end{align*}
 and their adjoints by
\begin{align*}
   \ES^* \colon \R^{n} \to \RKHS \quad \text{ s.t. } w \mapsto \tfrac{1}{\sqrt{n}}\sum_{i\in[n]}w_i\fH(x_i) & \quad \text{ and } & \EZ^* \colon \R^{n} \to \RKHS \quad \text{ s.t. } w \mapsto \tfrac{1}{\sqrt{n}}\sum_{i\in[n]}w_i\fG(y_i),
\end{align*}
we obtain  $\ERisk^\reg(\Estim){=}\hnorm{\EZ {-} \ES \Estim}^2 + \reg\hnorm{\Estim}^2$.

In the following we also use the empirical covariance operators defined as
\begin{equation}\label{eq:empirical_cov}
\ECx := \ES ^*\ES =\tfrac{1}{n} \sum_{i\in[n]}\phi(x_i)\otimes \phi(x_i) \;\text{ and }\;\ECy :=  \EZ ^*\EZ = \tfrac{1}{n} \sum_{i\in[n]}\phi(y_i)\otimes \phi(y_i),
\end{equation}
empirical cross-covariance operator
\begin{equation}\label{eq:empirical_cross-cov}
\ECxy := \ES ^*\EZ = \tfrac{1}{n} \sum_{i\in[n]}\phi(x_i)\otimes \phi(y_i),
\end{equation}
and kernel Gramm matrices
\begin{equation}\label{eq:kernel_mx}
\Kx := \ES \ES^* =\tfrac{1}{n} [k(x_i,x_j)]_{i,j\in[n]}\in\R^{n\times n} \;\text{ and }\; \Ky := \EZ \EZ^* =\tfrac{1}{n} [k(y_i,y_j)]_{i,j\in[n]}\in\R^{n\times n}.
\end{equation}
Additionally, we let $\ECx_\reg : = \ECx+\reg \Id_{\RKHS}$ be the regularized empirical covariance and $\Kx_{\reg} := \Kx + \reg \Id_n$ the regularized kernel Gram matrix. Then
we obtain the empirical estimators of the Koopman operator on an RKHS that correspond to the population ones:
empirical KRR estimator $\ERKoop:=\ECreg^{-1} \ECxy$, empirical PCR estimator $\SVDr{\ECx}^\dagger \ECxy$, and empirical RRR estimator $\ECreg^{-1/2} \SVDr{ \ECreg^{-1/2} \ECxy }$.

Noting that all of the empirical estimators above are of the form $ \EEstim = \ES U_r V_r^\top \EZ$,  where $U_r,V_r \in\R^{n\times r}$ and $r\in[n]$, see \cite{Kostic2022}, we conclude this section with the result on how their spectral decompositions can be computed in an infinite dimensional RKHS.

\begin{theorem}[\cite{Kostic2022}]\label{thm:spec_decomp_e}
Let $1\leq r \leq n$ and $ \EEstim = \ES U_r V_r^\top \EZ$,  where $U_r,V_r \in\R^{n\times r}$. If $V_r^\top \Kyx U_r \in\R^{r\times r}$, for $\Kyx = n^{-1}[k(y_i,x_j)_{i,j\in[n]}]$, is full rank and non-defective, the spectral decomposition $(\eeval_i,\elefun_i, \erefun_i)_{i\in[r]}$ of $\EEstim$ can be expressed in terms of the spectral decomposition $(\eeval_i,\levec_i, \revec_i)_{i\in[r]}$~of~ $V_r^\top \Kyx U_r $ as $\elefun_i = \eeval_i \EZ^*V_r \levec_i / \abs{\eeval_i}$ and $\erefun_i = \ES^*U_r \revec_i$, for all $i\in[r]$.
\end{theorem}

\section{Approach}\label{app:approach}

In this section, we prove key perturbation result and discuss the properties of the metric distortion. 
We conclude this section with the approximation bound for arbitrary estimator $\Estim\in\HSr$ that is the basis of the statistical bounds that follow. This result is a direct consequence of \cite{Kostic2022} and Davis-Khan spectral perturbation result for compact self-adjoint operators, \cite{DK1970}.

\thmSpPert*
\begin{proof}
We first remark that 
\[
\norm{(\Koop-\eeval_i\,I_{\Lii})^{-1}}^{-1} \leq \norm{(\Koop\TS-\TS\EEstim)\erefun_i} / \norm{\TS\erefun_i}\leq\error(\EEstim)\metdist(\erefun_i).
\]
Then, from the first inequality, using that $\Koop$ is normal, we obtain the first bound in \eqref{eq:bound_eval}. So, observing that for every $\keval\in\Spec(\Koop)\setminus\{\keval_i\}$, 
\[
\abs{\eeval_i-\keval} \geq \abs{\keval_i-\keval} - \abs{\eeval_i-\keval_i} \geq \abs{\keval_i-\keval} - \error(\EEstim)\,\metdist(\erefun_i),  
\]
we conclude that 
\[
\min\{\abs{\eeval_i-\keval}\,\vert\,\keval\in\Spec(\Koop)\setminus\{\keval_i\}\} \geq \gap_i(\Koop) - \error(\EEstim)\,\metdist(\erefun_i).
\]
So, applying Proposition~\ref{prop:davis_kahan}, we obtain
\begin{align*}
\sin(\sphericalangle(\ekefun_i,\kefun_i))& \leq \frac{\norm{\Koop\ekefun_i-\eeval_i\,\ekefun_i}}{[\gap_i(\Koop) - \error(\EEstim)\,\metdist(\erefun_i)]_+} \leq \frac{\norm{(\Koop\TS-\TS\EEstim)\erefun_i} / \norm{\TS\erefun_i}}{[\gap_i(\Koop) - \error(\EEstim)\,\metdist(\erefun_i)]_+} \\
& \leq \frac{\error(\EEstim)\,\metdist(\erefun_i)}{[\gap_i(\Koop) - \error(\EEstim)\,\metdist(\erefun_i)]_+}.
\end{align*}
Since, clearly  $\norm{\ekefun_i-\kefun_i}^2\leq 2(1-\cos(\sphericalangle(\ekefun_i,\kefun_i))\leq 2\sin(\sphericalangle(\ekefun_i,\kefun_i))$, the proof of the second bound in \eqref{eq:bound_eval} is completed.
\end{proof}

The following result shows how for the finite rank estimators one can control the metric distortion.
\propMetDist*
\begin{proof}
First, we have that $\erefun_i = \eeval_i^{-1} \EEstim\erefun_i = \eeval_i^{-1} \EEstim\EEstim^\dagger \EEstim\erefun_i$. But, then $g_i := \EEstim^\dagger \EEstim\erefun_i \in\Ker(\EEstim)^\perp$ and $ \eeval_i \erefun_i = \EEstim g_i $. 

Next, recall that 
\begin{equation*}
\inf_{g\in  \Ker(\Cx^{1/2} \EEstim)^\perp }\frac{\norm{ \Cx^{1/2} \EEstim g}}{\norm{g}} = \sigma_{\min}^+(\Cx^{1/2} \EEstim) = \sigma_r(\TS \EEstim).
\end{equation*}
So, since $\range(\EEstim^*\Cx^{1/2})\subseteq \range(\EEstim^*)$, then $g_i\in\Ker(\EEstim)^\perp \subseteq \Ker(\Cx^{1/2} \EEstim)^\perp$, and we conclude
\[
\sigma_r(\TS \EEstim)  \norm{g_i} \leq \norm{\Cx^{1/2}\EEstim\erefun_i} = \abs{\eeval_i } \norm{\TS\erefun_i}.
\]
Therefore, since
\[
\norm{g_i}^2 = \scalarp{\erefun_i,g_i}=\cos(\sphericalangle(\erefun_i,\Ker(\EEstim)^\perp))\norm{\erefun_i}\norm{g_i} = \cos(\sphericalangle(\erefun_i,\range(\EEstim^*)))\,\norm{\erefun_i}\norm{g_i},
\]
we have that
\[
\frac{\abs{\eeval_i } \norm{\TS\erefun_i}}{\sigma_r(\TS \EEstim)} \geq \norm{g_i} = \cos(\sphericalangle(\erefun_i,\range(\EEstim^*)))\,\norm{\erefun_i} \geq \left( \abs{\cos(\sphericalangle(\erefun_i,\elefun_i))} \wedge \abs{\eeval_i}\,\norm{\EEstim}^{-1}\right)\,\norm{\erefun_i}
\]
where the last inequality holds since $\elefun_i\in\range(\EEstim^*)$ and
\[
\norm{g_i}= \abs{\eeval_i}^{-1}\norm{\EEstim g_i} = \abs{\eeval_i}^{-1}\norm{\EEstim \erefun_i} \leq \abs{\eeval_i}^{-1} \norm{\EEstim}\norm{\erefun_i}.
\]
We remark that this inequality becomes equality when the eigenvalue $\eeval_i$ is simple. Finally, noticing that $\abs{\cos(\sphericalangle(\erefun_i,\elefun_i))}=\cond(\eeval_i)$ we have 
\[
\metdist(\erefun_i) \leq \frac{\abs{\eeval_i}\, \cond(\eeval_i) \wedge \norm{\EEstim}}{\sigma_{r}(\TS\EEstim)},
\]
and application of Weyl's inequality to the denominator completes the proof.
\end{proof}

\begin{remark}\label{rem:metdist_tight}
We remark that in Example~\ref{ex:OU} of the main text, for the choice of kernel with $\Pi\colon i\mapsto i$, $i\in\N$, after some basic algebra, for the $r$-th eigenpair $(\eval_r,\refun_r)$ of $\RRR$ we obtain that $\metdist(\refun_r) = \eval_r \cond(\refun_r) / \sigma_r(\TS\RRR)$. This makes inequality \eqref{eq:metric_dist_bound} tight. Moreover, for small enough $\reg>0$ and $r=1$, $\abs{\eval_r - \keval_r}=\metdist(\refun_r)\error(\RRR)$, where . Hence,  \eqref{eq:bound_eval} is also tight.
\end{remark}

Next result provides the reasoning for using empirical metric distortion given by \eqref{eq:emp_met_dist_main} in the main body.

\begin{proposition}
\label{prop:emp_met_dist}
Given $r\in\N$, let $(\eeval,\erefun_i)_{i=1}^r$ be nonzero eigenpairs of $\EEstim = \ES^*U_rV_r^\top\EZ \in\HSr$. If $(\revec_i)_{i\in[r]}$ are eigenvectors of the non-defective matrix $V_r^\top \Kyx U_r \in\R^{r\times r}$, for $\Kyx = n^{-1}[k(y_i,x_j)_{i,j\in[n]}]$, then for every $i\in[r]$
\begin{equation}\label{eq:emp_met_dist}
\emetdist_i= \frac{\norm{\erefun_i}}{ \norm{\ES \erefun_i}} = \sqrt{\frac{ \revec_i^*U_r^\top \Kx U_r \revec_i}{\norm{\Kx U_r \revec_i}^2}},
\end{equation}
and
\begin{equation}\label{eq:emp_met_dist_con}
\left\vert \emetdist_i - \metdist(\erefun_i)  \right\vert \leq \left( \metdist(\erefun_i)\;\wedge \emetdist_i\right) \,\metdist(\erefun_i)\, \emetdist_i\,\norm{\ECx-\Cx}.
\end{equation}
\end{proposition} 
\begin{proof}
First, note that \eqref{eq:emp_met_dist} follows directly from Theorem~\ref{thm:spec_decomp_e}. Next, since for every $i\in[r]$,
\[
(\emetdist_i)^{-2} - (\metdist(\erefun_i))^{-2} = \frac{\scalarp{\erefun_i,(\ECx-\Cx)\erefun_i)}}{\norm{\erefun_i}^2} \leq \norm{\ECx-\Cx},
\]
we obtain 
\[
\left\vert \emetdist_i^{-1} - (\metdist(\erefun_i))^{-1}  \right\vert \leq \frac{\left\vert \emetdist_i^{-2} - (\metdist(\erefun_i))^{-2}  \right\vert}{(\metdist(\erefun_i))^{-1} \;\vee\; \emetdist_i^{-1}}\leq \left( \metdist(\erefun_i)\;\wedge \emetdist_i\right) \norm{\ECx-\Cx}.
\]
\end{proof}

\begin{remark}\label{rem:spurious}
    We remark that for deterministic dynamical systems, to check if an eigenpair $(\eeval_i,\erefun_i)$ is spurious authors in \cite{Kutz2016,Colbrook2021} suggest to check if the $\erefun_i(y_i)\approx \eeval_i \erefun_i(x_i)$ on a training set $\Data$. Clearly, one should check the same on the validation set in order to assure that over-fitting does not occur. It is interesting to note that such strategies rely on the empirical estimate $\norm{(\EZ-\ES\EEstim)\erefun_i}$, while our analysis aims to give high-probability finite sample guarantees via the bounds on metric distortion and operator norm error.
\end{remark}

\section{Controlling the Operator Norm Error}\label{app:error}

\subsection{Main Assumptions}\label{app:assumptions}

We start by observing that $\TS\in\HS{\RKHS,\Lii}$, according to the spectral theorem for positive self-adjoint operators, has an SVD, i.e. there exists at most countable positive sequence $(\sigma_j)_{j\in J}$, where $J:=\{1,2,\ldots,\}\subseteq\N$, and ortho-normal systems $(\ell_j)_{j\in J}$ and $(h_j)_{j\in J}$ of $\cl(\range(\TS))$ and $\Ker(\TS)^\perp$, respectively, such that $\TS h_j = \sigma_j \ell_j$ and $\TS^* \ell_j = \sigma_j h_j$, $j\in J$. 

Now, given $\rpar\geq 0$, let us define scaled injection operator $\TS_{\rpar} \colon \RKHS \to \Lii$ as
\begin{equation}\label{eq:injection_scaled}
\TS_\rpar:= \sum_{j\in J}\sigma_j^{\rpar}\ell_j\otimes h_j.
\end{equation}
Clearly, we have that $\TS = \TS_1$, while $\range{\TS_0} = \cl(\range(\TS))$. Next, we equip $\range(\TS_{\rpar})$ with a norm $\norm{\cdot}_\rpar$ to build an interpolation space. 
\[
[\RKHS]_\rpar:=\left\{ f\in\range(\TS_{\rpar})\;\vert\; \norm{f}_\rpar^2:= \sum_{j\in J}\sigma_j^{-2 \rpar} \scalarp{f,\ell_j}^2 <\infty \right\}.
\]

We remark that for $\rpar=1$ the space $[\RKHS]_\rpar$ is just an RKHS $\RKHS$ seen as a subspace of $\Lii$.  Moreover, we have the following injections
\[
 [\RKHS]_{\rpar_1} \hookrightarrow [\RKHS]_1 \hookrightarrow [\RKHS]_{\rpar_2}  \hookrightarrow [\RKHS]_{0}  = \Lii,
\]
where $\rpar_1\geq1\geq\rpar_2\geq0$.

In addition, from \ref{eq:BK} we also have that RKHS $\RKHS$ can be embedded into $L^{\infty}_\im(\X)$, i.e.  for some $\epar\in(0,1]$
\[
 [\RKHS]_1 \hookrightarrow [\RKHS]_{\epar}  \hookrightarrow L^{\infty}_\im(\X) \hookrightarrow \Lii,
\]
Now, according to \cite{Fischer2020}, if $\TS_{\epar,\infty}\colon [\RKHS]_\epar \hookrightarrow L^{\infty}_\im(\X)$ denotes the injection operator, its boundedness implies the polynomial decay of the singular values of $\TS$, i.e. $\sigma_j^2(\TS)\lesssim j^{-1/\epar}$, $j\in J$, and the following condition is assured
\begin{enumerate}[label={\rm \textbf{(KE)}},leftmargin=15ex]
\item\label{eq:KE} \emph{Kernel embedding  property}: there exists $\epar\in[\spar,1]$ such that 
\begin{equation}\label{eq:c_beta}
c_{\epar}:=\norm{\TS_{\epar,\infty}}^2 =\esssup_{x\sim\im}\sum_{j\in J}\sigma^{2\epar}_j|\ell_j(x)|^2 <+\infty.
\end{equation}
\end{enumerate}

In what follows we discuss how our novel assumption \ref{eq:RC} compares to the existing ones, quantifying how miss-specified the learning problem is. We consider the following assumption made in \cite{Li2022} to analyze the HS error of CME
\begin{enumerate}[label={\rm \textbf{(SRC)}},leftmargin=15ex]
\item\label{eq:SRC} \emph{Source condition from \cite{Li2022}}: for for some $\rpar\in(0,2]$
\[
\range(\TZ) \subseteq \range(\TS_\rpar)\;\text{ and }\; \Estim_\RKHS^\rpar:=\TS_\rpar^\dagger \Cxy\in\HS{\RKHS}.
\]
\end{enumerate}

\begin{remark}[Finite-dimensional RKHS]
    When $\RKHS$ is finite dimensional, all spaces $[\RKHS]_\rpar$ are finite dimensional. Hence, $\range(\TZ)\subset \range(\TS)$ implies also $\range(\TZ)\subset \range(\TS_\rpar)$ for every $\rpar>0$. Moreover, we can set $\epar$ arbitrary close to zero.
\end{remark}

\VK{
\begin{remark}[\ref{eq:SRC} vs. \ref{eq:RC}]\label{rem:src_implies_rc}
According to \cite[Theorem 2.2]{zabczyk2020}, the condition \ref{eq:RC} is equivalent to $\range(\TZ) \subseteq \range(\TS_{\rpar})$, i.e. $\HKoop^\rpar:=\TS_{\rpar}^\dagger \Cxy$ is bounded operator on $\RKHS$ and $\TZ=\TS_\rpar \HKoop^\rpar$. Hence, if \ref{eq:SRC} holds for some $\rpar\in(0,2]$, then \ref{eq:RC} holds, too. Indeed, we have that $\Cxy = \TS^*\TZ = \TS^*\TS_\rpar \Estim_\RKHS^\rpar$, and, thus, $\Cxy\Cxy^*\preceq \norm{\Estim_\RKHS^\rpar}^2 \Cx^{1+\rpar}$. On the other hand, for the Koopman operator $\Koop = I_{\Lii}$, of Example~\ref{ex:rank1_koop}, while $\HKoop^1=I_{\RKHS}$ implies \ref{eq:RC} for at least one $\rpar\in[1,2]$, one can show that \ref{eq:SRC} doesn't hold for any $\rpar>0$, c.f.~\cite[Appendix D]{Li2022}.
\end{remark}
}

\subsection{Bounding the Bias}\label{app:bias}

Recalling the decomposition
\begin{equation}\label{eq:error_decomp}
\underbrace{\error(\EEstim):=\norm{\TZ - \TS\EEstim}}_{\text{operator norm error}} \;\;\leq \underbrace{\norm{\TZ - \TS\RKoop}}_{\text{bias due to regularization}} + \underbrace{\norm{\TS (\RKoop - \Estim)}}_{\text{bias due to rank reduction}}  + \underbrace{\norm{\TS(\Estim - \EEstim)}.}_{\text{variance of the estimator}}    
\end{equation}
we first prove the bound of the first term.

\begin{proposition}\label{prop:app_bound}
Let $\RKoop=\Creg^{-1}\Cxy$ for $\reg >0$, and $P_\RKHS\colon\Lii\to\Lii$ be the orthogonal projector onto $\cl(\range(\TS))$. If the assumptions \ref{eq:BK}, \ref{eq:SD} and \ref{eq:RC}  hold, then $\norm{\RKoop} \leq \rcon \bcon^{(\rpar-1)/2}$ for $\rpar\in[1,2]$, $\norm{\RKoop}\leq \rcon\,\reg^{(\rpar-1)/2}$ for $\rpar\in(0,1]$, and
\begin{equation}\label{eq:boundRC}
\norm{\Koop\TS - \TS \RKoop} \leq \rcon\, \reg^{\frac{\rpar}{2}}+\norm{(I-P_\RKHS)\Koop\TS}.
\end{equation}
\end{proposition}
\begin{proof}
Recalling that $P_\RKHS:=\sum_{j\in J}\ell_j\otimes\ell_j$, start by denoting the orthogonal projectors on the subspace of $k$ leading left singular functions of $\TS$ as   $P_k:=\sum_{j\in [k]}\ell\otimes\ell$, respectively. Next, observe that 
\begin{align*}
\TZ - \TS \RKoop & = (I_{\Lii} - \TS \Creg^{-1}\TS^*)\TZ = (I_{\Lii} - (\TS \TS^* + \reg I_{\RKHS})^{-1} \TS\TS^*)\TZ \\
& = \reg (\TS \TS^* + \reg I_{\RKHS})^{-1}\TZ = \left( \sum_{j\in J}\frac{\reg}{\sigma_j^2+\reg}\ell_j\otimes\ell_j \right)\TZ =  \left( \sum_{j\in J}\frac{\reg}{(\sigma_j^2+\reg)\sigma_j}\ell_j\otimes h_j \right)\Cxy.
\end{align*}
Therefore, for every $k\in J$
\[
\norm{P_k(\TZ - \TS \RKoop)}^2 = \left\Vert \left( \sum_{j\in [k]}\frac{\reg}{(\sigma_j^2+\reg)\sigma_j}\ell_j\otimes h_j \right)\Cxy \Cxy^* \left( \sum_{j\in [k]}\frac{\reg}{(\sigma_j^2+\reg)\sigma_j}h_j\otimes \ell_j \right) \right\Vert,
\]
which, due to \ref{eq:RC}, implies that
\[
\norm{P_k(\TZ - \TS \RKoop)} \leq \rcon\,\left\Vert  \sum_{j\in [k]}\frac{\reg\,\sigma_j^\rpar}{\sigma_j^2+\reg}\ell_j \otimes\ell_j  \right\Vert. 
\]
On the other hand, 
\[
\sum_{j\in [k]} \frac{\reg\,\sigma_j^\rpar}{\sigma_j^2 +\reg} \ell_j \otimes \ell_j = \reg^{\frac{\rpar}{2}} \sum_{j\in [k]} \frac{(\sigma_j^2 \reg^{-1})^{\frac{\rpar}{2}}}{\sigma_j^2\reg^{-1} + 1} \ell_j \otimes \ell_j \preceq \reg^{\frac{\rpar}{2}} \sum_{j\in [k]} \ell_j \otimes \ell_j,
\]
where the inequality holds due to $x^s \leq x+1$ for all $x\geq0$ and $s\in[0,1]$. Since the norm of the projector equals one, we get $\norm{P_k(\TZ - \TS \RKoop)}\leq \rcon \reg^{\frac{\rpar}{2}}$.

Next, observe that 
\begin{align*}
\norm{(P_\RKHS-P_k)(\TZ - \TS \RKoop)}^2 & =  \left\Vert  \sum_{j\in J\setminus[k]}\frac{\reg^2}{(\sigma_j^2+\reg)^2}(\TZ^*\ell_j)\otimes (\TZ^*\ell_j)  \right\Vert  \leq \sum_{j\in J\setminus[k]} \frac{\reg^2}{(\sigma_{j}^2+\reg)^2}\norm{\TZ^*\ell_{j}}^2 \\
& \leq \sum_{j\in J\setminus[k]} \frac{\reg^2\,\sigma_{j}^{2\rpar}}{(\sigma_{j}^2+\reg)^2} \leq \sum_{j\in J\setminus[k]} \sigma_{j}^{2\rpar}
\end{align*}
So, using triangular inequality, for every $k\in J$ we have 
\[
\norm{P_\RKHS(\TZ-\TS\RKoop)} \leq \norm{P_k(\TZ-\TS\RKoop)} + \norm{(P_\RKHS-P_k)(\TZ-\TS\RKoop)} \leq \rcon \reg^{\frac{\rpar}{2}} + \sum_{j\in J\setminus[k]} (\sigma_{j}^{2\spar})^{\frac{\rpar}{\spar}},
\]
and, hence, letting $k\to\infty$ we obtain $\norm{P_\RKHS\TZ-\TS\RKoop} \leq \rcon \reg^{\frac{\rpar}{2}}$. Hence, \eqref{eq:boundRC} follows from triangular inequality.

\medskip

To estimate the $\norm{\RKoop}$, note that \ref{eq:RC} implies $\norm{\RKoop}\leq\,\rcon\,\norm{\Creg^{-1}\Cx^{\frac{1+\rpar}{2}}}$ and considering two cases. First, if \ref{eq:RC} holds for some $\rpar\in[1,2]$, then, clearly $\norm{\RKoop} \leq \rcon \bcon^{(\rpar-1)/2}$. On the other hand, if $\rpar\in(0,1]$, then 
\[
\frac{\sigma_j^{1+\rpar}}{\sigma_j^2+\reg} = \reg^{-1}\frac{\left(\sigma_j^2\reg^{-1}\right)^{\frac{1+\rpar}{2}}}{\sigma_j^2\,\reg^{-1}+1}\leq \reg^{\frac{\rpar-1}{2}},
\]
and, thus,  $\norm{\RKoop}\leq \rcon\,\reg^{(\rpar-1)/2}$.

\end{proof}

\begin{remark}\label{rem:app_bound}
Inequality \eqref{eq:boundRC} says that the regularization bias is comprised of a term depending on  the choice of $\reg$, 
and on a term depending on the ``alignment'' between $\RKHS$ and $\range(\Koop\TS)$. The term $\norm{(I-P_\RKHS)\Koop\TS}$ can be set to zero by two different approaches. One is choose 
a kernel which in some way minimizes $\norm{(I-P_\RKHS)\Koop\TS}$. Another is to choose a universal kernel~\citep[Chapter 4]{Steinwart2008}, for which $\range(\Koop\TS) \subseteq \cl(\range(\TS))$. While the former approach is common to several methods using finite-dimensional kernels, see e.g. \cite{Wu2019,Mardt2018,Wang2022, Bonati2021}, the latter is classical in kernel-based learning of the Koopman operator, \cite{LGBPP12,Li2022,Kostic2022}.    
\end{remark}

In order to proceed with bounding the bias due to rank reduction for both considered estimators, we first provide auxiliary result.

\begin{proposition}\label{prop:svals_app_bound}
Let $B:=\Creg^{-1/2} \Cxy$, let \ref{eq:RC} hold for some $\rpar\in(0,2]$. Then for every $j\in J$
\begin{equation}\label{eq:svals_app_bound}
\sigma_j^2(\TZ)-\rcon^2\,\bcon^{\rpar/2}\,\reg^{\rpar/2} \leq \sigma_j^2(\TB) \leq \sigma_j^2(\TZ).
\end{equation}
\end{proposition}
\begin{proof}
Start by observing that 
\[
\TB^*\TB = \TZ^*\TS \Creg^{-1}\TS^*\TZ = \TZ^*\TZ - \reg \TZ^*(\TS \TS^*+ \reg I_{\Lii})^{-1} \TZ,
\]
implies that
\[
\TZ^*\TZ-\sum_{j\in J}\frac{\reg}{\sigma_j^2+\reg} (\TZ^*\ell_j)\otimes(\TZ^*\ell_j) = \TB^*\TB \preceq \TZ^*\TZ.
\]

Next, similarly to the above, for every $k\in J$, we have 
\begin{align*}
\left\Vert\sum_{j\in[k]}\frac{\reg}{\sigma_j^2+\reg} (\TZ^*\ell_j)\otimes(\TZ^*\ell_j) \right\Vert & \leq \rcon^2 \left\Vert\sum_{j\in[k]}\frac{\sigma_j^{2\rpar}}{\sigma_j^2\reg^{-1}+1} \ell_j\otimes\ell_j \right\Vert \\
& = \rcon^2 \left\Vert\sum_{j\in[k]}\frac{(\sigma_j^2\reg^{-1})^{\rpar/2} \sigma_j^{\rpar} \reg^{\rpar/2}}{\sigma_j^2\reg^{-1}+1} \ell_j\otimes\ell_j \right\Vert \leq \rcon^2 \reg^{\rpar/2} \norm{\Cx}^{\rpar/2},    
\end{align*}
and
\[
\left\Vert\sum_{j\in J\setminus [k]}\frac{\reg}{\sigma_j^2+\reg} (\TZ^*\ell_j)\otimes(\TZ^*\ell_j) \right\Vert \leq \rcon^2\sum_{j\in J\setminus [k]}\frac{\reg}{\sigma_j^2+\reg} \sigma_j^{2\rpar}\leq \rcon^2\sum_{j\in J\setminus [k]} (\sigma_j^{2\spar})^{\rpar/\spar}.
\]
So, as before, letting $k\to\infty$ we get the result
\end{proof}

Now, the bounds for the rank reduction bias of the two estimator follow.

\begin{proposition}[RRR]\label{prop:rrr_bias}
Let \ref{eq:RC} hold for some $\rpar\in(0,2]$. Then the bias of $\RRR$ due to rank reduction is bounded as
\begin{equation}\label{eq:rrr_bias}
\sigma_{r+1}(\TZ) - \rcon\, \bcon^{\rpar/4}\reg^{\rpar/4} - 2\,\rcon\,\reg^{(1\wedge \rpar)/2} \leq \norm{\TS(\RKoop-\RRR)} \leq \sigma_{r+1}(\TZ). 
\end{equation}
\end{proposition}
\begin{proof} Observe that
\[
\norm{\TS(\RKoop-\RRR)}\leq \norm{\Creg^{1/2}(\RKoop-\RRR)} = \norm{\TB-\SVDr{\TB}} = \sigma_{r+1}(\TB) \leq  \sigma_{r+1}(\TZ) 
\]
while 
\begin{align*}
    \norm{\TS(\RKoop-\RRR)} & \geq \norm{\Creg^{1/2}(\RKoop-\RRR)} -\reg^{1/2} \norm{\RKoop-\RRR} \\ 
    &\geq \sigma_{r+1}(\TZ) - \rcon \norm{\Cx}^{\rpar/4}\reg^{\rpar/4} - 2\rcon\reg^{(1\wedge \rpar)/2}. 
\end{align*}

\end{proof}

\begin{proposition}[PCR]\label{prop:pcr_bias}
Let \ref{eq:RC} hold for some $\rpar\in(0,2]$. Then the bias of $\PCR$ due to rank reduction is bounded as
\begin{equation}\label{eq:pcr_bias}
\sigma_{r+1}(\TS) - \sqrt{\IrRisk^{r+1}} - \rcon\,\reg^{\rpar/2} \leq \norm{\TS(\RKoop-\PCR)} \leq \sigma_{r+1}(\TS), 
\end{equation}
where $\IrRisk^{r+1}:= \scalarp{(\TS\TS^* - \TZ\TZ^*)\ell_{r+1},\ell_{r+1}}\geq0$ is the irreducible risk restricted to the $(r+1)$-st left singular function of $\TS$.
\end{proposition}
\begin{proof}
Let $\TP_r$ denote the orthogonal projector onto the subspace of leading $r$ eigenfunctions of $\Cx$. Then the upper bound is easily obtained as
\[
\norm{\TS(\RKoop-\PCR)}= \norm{\Cx^{1/2}(I-\TP_r)\RKoop}\leq \sigma_{r+1}(\TS) \norm{(I-\TP_r)\RKoop} \leq \sigma_{r+1}(\TS).
\]
Next, observe that
\begin{align*}
\norm{\TS(\RKoop-\PCR)}^2 & = \norm{\TZ^*\TS(I-\TP_r)\Cx\Creg^{-2}\TS^* \TZ} = \left\Vert \sum_{j\geq r+1} \frac{\sigma_j^4}{(\sigma_j^2+\reg)^2} (\TZ^*\ell_j)\otimes (\TZ^*\ell_j)\right\Vert \\
&\geq \left\Vert \frac{\sigma_{r+1}^4}{(\sigma_{r+1}^2+\reg)^2} (\TZ^*\ell_j)\otimes (\TZ^*\ell_j)\right\Vert = \frac{\sigma_{r+1}^4}{(\sigma_{r+1}^2+\reg)^2}\norm{\TZ^*\ell_{r+1}}^2.     
\end{align*}
But, since $\norm{\TZ^*\ell_{r+1}} = \norm{\Cxy h_{r+1}} / \sigma_{r+1} \leq \rcon\norm{\Cx^{(1+\rpar)/2} h_{r+1}} / \sigma_{r+1} = \rcon\,\sigma_{r+1}^{\rpar}$,
we obtain 
\begin{align*}
\norm{\TS(\RKoop-\PCR)} & \geq \left(1 - \frac{\reg}{\sigma_{r+1}^2+\reg}\right)\norm{\TZ^*\ell_{r+1}}\geq \norm{\TZ^*\ell_{r+1}} - \rcon \frac{\reg \sigma_{r+1}}{\sigma_{r+1}^2+\reg} \\
& = \norm{\TZ^*\ell_{r+1}} - \frac{(\sigma_{r+1}^2\reg^{-1})^{\rpar/2}}{\sigma_{r+1}^2\reg^{-1}+1} \geq  \norm{\TZ^*\ell_{r+1}} - \rcon\,\reg^{\rpar/2}.
\end{align*}
Finally,
\[
\norm{\TZ^*\ell_{r+1}}^2 = \scalarp{\TZ\TZ^*\ell_{r+1},\ell_{r+1}} = \scalarp{\TS\TS^*\ell_{r+1},\ell_{r+1}} - \IrRisk^{r+1},
\]
which completes the proof
\end{proof}

We observe that $\IrRisk^{r+1}$ measures the variance of the Koopman operator $\Koop$ over the $(r+1)$-st left singular function of $\TS$. Hence, it is immediate that $\IrRisk^{r+1}\leq \IrRisk$, and the previous result indicates that when Koopman operator has small irreducible risk the bias introduced by PCR is indeed of the order $\sigma_{r+1}(\TS)$. %
Namely, the bound in \eqref{eq:pcr_bias} is sharp provided that $\sigma_{r+1}(S)\geq (1+c) \sqrt{\IrRisk^{r+1}}
$ for some absolute constant $c>0$. This means that the learning rate of PCR can be significantly worse than that of RRR in bias dominating scenarios which can occur when the kernel is somehow "misaligned" with the Koopman operator (that is $\sigma_{r+1}(\TS)\gg \sigma_{r+1}(\TZ)\vee n^{-\frac{\alpha}{2(\alpha+\beta)}} $). We have provided one such example in Example~\ref{ex:OU}.

\subsection{Bounding the Variance}\label{app:variance}

\subsubsection{Concentration Inequalities}\label{app:concentration_ineq}

All the statistical bounds we present will relay on two versions of Bernstein inequality. The first one is Pinelis and Sakhanenko inequality for random variables in a separable Hilbert space, see \citep[][Proposition 2]{caponnetto2007}.

\begin{proposition}\label{prop:con_ineq_ps}
Let $A_i$, $i\in[n]$ be i.i.d copies of a random variable $A$ in a separable Hilbert space with norm $\norm{\cdot}$. If there exist constants $L>0$ and $\sigma>0$ such that for every $m\geq2$ $\EE\norm{A}^m\leq \frac{1}{2} m! L^{m-2} \sigma^2$, 
then with probability at least $1-\delta$ 
\begin{equation}\label{eq:con_ineq_ps}
\left\|\frac{1}{n}\sum_{i\in[n]}A_i - \EE A \right\|\leq \frac{4\sqrt{2}}{\sqrt{n}} \log\frac{2}{\delta} \sqrt{ \sigma^2 + \frac{L^2}{n} }
\end{equation}
\end{proposition}

\noindent On the other hand, we recall that in \cite{minsker2017}, a dimension-free version of the non-commutative Bernstein inequality for finite-dimensional symmetric matrices is proposed (see also Theorem 7.3.1 in \cite{tropp2012user} for an easier to read and slightly improved version) as well as an extension to self-adjoint Hilbert-Schmidt operators on a separable Hilbert spaces.%

\begin{proposition}\label{prop:con_ineq_0}
Let $A_i$, $i\in[n]$ be i.i.d copies of a Hilbert-Schmidt operator $A$ on the separable Hilbert space. Let $\norm{A}\leq c$ almost surely, $\EE A =0$ and let $\EE[A^2]\preceq V$ for some  trace class operator $V$. Then with probability at least $1-\delta$ 
\begin{equation}\label{eq:con_ineq_0}
\left\|\frac{1}{n}\sum_{i\in[n]}A_i \right\|\leq \frac{2c}{3n} \mathcal{L}_A(\delta)+ \sqrt{\frac{2\norm{V}}{n}\mathcal{L}_A(\delta)},
\end{equation}
where
\[
\mathcal{L}_A(\delta):= \log\frac{4}{\delta}+ \log\frac{\tr(V)}{\norm{V}} 
\]
\end{proposition}

We use the same strategy in combination with a standard dilation method to extend a deviation inequality on rectangular matrices of \cite{tropp2012user} (Corollary 7.3.2) to Hilbert-Schmidt operators on a separable Hilbert space.

\begin{proposition}\label{prop:con_ineq}
Let $A_i$, $i\in[n]$ be i.i.d copies of a Hilbert-Schmidt operator $A$ on the separable Hilbert space. Let $\norm{A}\leq c$ almost surely and let $\EE[AA^*]\preceq V$ and $\EE[A^*A]\preceq V'$ for some  trace class operators $V$ and $V'$. Then with probability at least $1-\delta$ 
\begin{equation}\label{eq:con_ineq}
\left\|\frac{1}{n}\sum_{i\in[n]}A_i - \EE A \right\|\leq \frac{4c}{3n} \mathcal{L}_A(\delta)+ \sqrt{\frac{2\max\{\norm{V},\norm{V'}\}}{n}\mathcal{L}_A(\delta)},
\end{equation}
where
\[
\mathcal{L}_A(\delta):= \log\frac{4}{\delta}+ \log\frac{\tr(V+V')}{\max\{\norm{V},\norm{V'}\}} 
\]
\end{proposition}
\begin{proof}
Let 
\[
B_i =
       \left[                      
        \begin{array}{cc}
         {0}  & {A_i}  \\
          {A_i^*}   & {0}  
         \end{array}
      \right] \text{ and } B =
       \left[                      
        \begin{array}{cc}
         {0}  & {A}  \\
          {A^*}   & {0}  
         \end{array}
      \right],
\]
then $\norm{B} =\norm{A}$, $\norm{\frac{1}{n}\sum_{i\in[n]}B_i-\EE B} = \norm{\frac{1}{n}\sum_{i\in[n]}A_i-\EE A}$ and 
\begin{equation*}
\EE (B - \EE B)^2\preceq \EE B^2 = \left[                      
        \begin{array}{cc}
         \EE [AA^*] & 0  \\
          0 &  \EE[A^* A]
         \end{array}
      \right] \preceq  \left[   
	\begin{array}{cc}
         V & 0  \\
          0 & V'
         \end{array}
      \right]=:V''.
\end{equation*}
Moreover, $\norm{B-\EE B}\leq \norm{B}+\sqrt{\EE\norm{B}^2}\leq 2c$.
Applying Proposition \ref{prop:con_ineq_0} we complete the proof. 
\end{proof}

\begin{proposition}\label{prop:cros_cov_bound}
Given $\delta>0$, with probability in the i.i.d. draw of $(x_i,y_i)_{i=1}^n$ from $\rho$, it holds that
\[
\PP\{ \norm{ \ECxy - \Cxy }\leq \rate_n(\delta) \} \wedge \PP\{ \norm{ \ECx - \Cx }\leq \rate_n(\delta) \} \geq 1-\delta,
\]
where
\begin{equation}\label{eq:eta}
\rate_n(\delta) := \frac{4\bcon}{3n} \mathcal{L}(\delta) + \sqrt{\frac{2\norm{\Cx}}{n}\mathcal{L}(\delta)}\quad\text{ and }\quad \mathcal{L}(\delta):= \log\frac{4\tr(\Cx)}{\delta\,\norm{\Cx}}.
\end{equation}

\end{proposition}
\begin{proof}
Proof follows directly from Proposition \ref{prop:con_ineq} applied to operators  $\phi(x_i)\otimes\phi(x_i)$ and  $\phi(x_i)\otimes\phi(y_i)$, respectively using the fact that $\Cx = \EE\,\phi(x_i)\otimes\phi(x_i) = \EE\,\phi(y_i)\otimes\phi(y_i)$ and $\Cxy = \EE \phi(x_i)\otimes\phi(y_i)$. 
\end{proof}

\begin{proposition}
\label{prop:bound_leftright}
Let \ref{eq:KE} hold for $\epar\in[\spar,1]$. Given $\delta>0$, with probability in the i.i.d. draw of $(x_i,y_i)_{i=1}^n$ from $\rho$, it holds that
\begin{equation}\label{eq:bound_leftright}
\PP\left\{ \norm{ \Creg^{-1/2} (\ECx - \Cx) \Creg^{-1/2}  } \leq \rate^1_n(\reg,\delta) \right\} \wedge \PP\left\{ \norm{ \Creg^{-1/2} (\ECxy - \Cxy)\Creg^{-1/2}  } \leq \rate^1_n(\reg,\delta) \right\}
\geq 1-\delta,    
\end{equation}
where 
\begin{equation}\label{eq:reg_eta1}
\rate_n^1(\reg,\delta) := \frac{4\econ}{3n\reg^{\epar}} \mathcal{L}^1(\reg,\delta)+ \sqrt{\frac{2\,\econ}{n\,\reg^{\epar}}\mathcal{L}^1(\reg,\delta)},
\end{equation}
where
\[
\mathcal{L}^1(\reg,\delta):=\ln \frac{4}{\delta} + \ln\frac{\tr(\Creg^{-1}\Cx)}{\norm{\Creg^{-1}\Cx}},. 
\]
Moreover, 
\begin{equation}\label{eq:bound_simetric}
\PP\left\{ \norm{ \Creg^{1/2} \ECreg^{-1} \Creg^{1/2}  } \leq \frac{1}{1-\rate_n^1(\reg,\delta)} \right\} 
\geq 1-\delta.
\end{equation}
\end{proposition}
\begin{proof}
The idea is to apply Proposition \ref{prop:con_ineq} for operator $A= \xi(x) \otimes \xi(x)$, where $\xi(x):=\Creg^{-1/2}\phi(x)$. To that end, observe that for every $\epar>0$ we have that
\begin{align*}
\norm{\xi(x)}^2 & = \sum_{j\in J}\scalarp{\Creg^{-1/2}\phi(x),h_j}^2 = \sum_{j\in J}\frac{1}{\sigma_j^2+\reg} \scalarp{\phi(x),h_j}^2 = \sum_{j\in J}\frac{\sigma_j^{2(1-\epar)}}{\sigma_j^2+\reg} \frac{\scalarp{\phi(x),h_j}^2}{\sigma_j^{2}} \sigma_j^{2\epar} \\ 
& = \reg^{-\epar}\sum_{j\in J}\frac{(\sigma_j^2\reg^{-1})^{1-\epar}}{\sigma_j^2\reg^{-1}+1} \frac{\abs{h_j (x)}^2}{\sigma_j^2} \sigma_j^{2\epar} \leq \reg^{-\epar}\sum_{j\in J} \frac{\abs{(\TS h_j)(x)}^2}{\sigma_j^2} \sigma_j^{2\epar} = \reg^{-\epar}\sum_{j\in J} \abs{\ell_j(x)}^2\sigma_j^{2\epar}
\end{align*}
So, due to \eqref{eq:c_beta}, we obtain $\norm{A}\leq\norm{\xi}_\infty^2 \leq \reg^{-\epar} c_{\epar}$.  On the other hand, since 
\[
\EE_{x\sim\im} (\xi(x)\otimes \xi(x))^2 \preceq \norm{\xi}_\infty^2 \EE_{x\sim\im} \xi(x)\otimes \xi(x) =  \norm{\xi}_\infty^2 \Creg^{-1/2} \Cx \Creg^{-1/2}.
\]

Next, we observe that
\[
\norm{I_\RKHS - \Creg^{-1/2} \ECreg \Creg^{-1/2}} = \norm{\Creg^{-1/2} (\Cx-\ECx) \Creg^{-1/2}}\leq \rate_n^1(\reg,\delta).
\]
Thus, for $\rate_n^1(\reg,\delta)$ smaller than one, it follows that
\[
\norm{ \Creg^{1/2} \ECreg^{-1} \Creg^{1/2}} =  \norm{(\Creg^{-1/2} \ECreg \Creg^{-1/2})^{-1}}\leq \frac{1}{1-\norm{I_\RKHS - \Creg^{-1/2} \ECreg \Creg^{-1/2}}},
\]
and the proof is completed.
\end{proof}

\begin{proposition}
\label{prop:bound_left}
Let \ref{eq:KE} hold for $\epar\in[\spar,1]$. Given $\delta>0$, with probability in the i.i.d. draw of $(x_i,y_i)_{i=1}^n$ from $\rho$, it holds
\[
\PP\left\{ \hnorm{ \Creg^{-1/2} (\ECx - \Cx) } \leq \rate^2_n(\reg,\delta) \right\} \wedge \PP\left\{ \hnorm{ \Creg^{-1/2} (\ECxy - \Cxy) } \leq \rate^2_n(\reg,\delta) \right\}
\geq 1-\delta,
\]
where
\begin{equation}\label{eq:reg_eta2}
\rate_n^2(\reg,\delta) := 4\,\sqrt{2\,\bcon}\,\ln\frac{2}{\delta}\,\sqrt{\frac{\tr(\Creg^{-1}\Cx)}{n} + \frac{\econ}{n^2\reg^{\epar}}}.
\end{equation}
\end{proposition}
\begin{proof}
First, recall that $\HS{\RKHS}$ equipped with $\hnorm{\cdot}$ is separable Hilbert space. Hence, we will apply Proposition \ref{prop:con_ineq_ps} for $A = \xi(x)\otimes\phi(y)$, where $\xi(x):=\Creg^{-1/2}\phi(x)$. To that end, 
observe that 
\begin{align}
\EE [\hnorm{A}^m] & = \EE \,[\norm{\xi(x)}^m\,\norm{\phi(y)}^m ]\leq  \norm{\xi}_\infty^{m-2}\norm{\phi}_\infty^{m}\,\EE\,[\norm{\xi(x)}^2] \\
&= \norm{\xi}_\infty^{m-2}\norm{\phi}_\infty^{m}\,\tr(\Creg^{-1}\Cx) \leq \frac{1}{2}m! \left( \reg^{-\epar/2}\,\sqrt{\econ\,\bcon} \right)^{m-2} \left(\sqrt{\bcon\,\tr(\Creg^{-1}\Cx)} \right)^2.  
\end{align}
\end{proof}

\begin{proposition}
\label{prop:bound_left2}
Let \ref{eq:KE} hold for $\epar\in[\spar,1]$. Given $\delta>0$, with probability in the i.i.d. draw of $(x_i,y_i)_{i=1}^n$ from $\rho$, it holds
\[
\PP\left\{ \hnorm{ \Creg^{-1} (\ECx - \Cx) } \leq \rate^3_n(\reg,\delta) \right\} \wedge \PP\left\{ \hnorm{ \Creg^{-1} (\ECxy - \Cxy) } \leq \rate^3_n(\reg,\delta) \right\}
\geq 1-\delta,
\]
where 
\begin{equation}\label{eq:reg_eta3}
\rate_n^3(\reg,\delta) := 4\,\sqrt{2\,\bcon}\,\ln\frac{2}{\delta}\,\sqrt{\frac{\tr(\Creg^{-1}\Cx)}{n} + \frac{\econ}{n^2\reg^{\epar+1}}}.
\end{equation}
\end{proposition}
\begin{proof}
Similar to the above, we apply Proposition \ref{prop:con_ineq_ps} for $A = \xi(x)\otimes\phi(y)$, where $\xi(x):=\Creg^{-1}\phi(x)$. Hence,
\begin{align*}
\norm{\xi(x)}^2 & = \sum_{j\in J}\scalarp{\Creg^{-1}\phi(x),h_j}^2 = \sum_{j\in J}\frac{1}{(\sigma_j^2+\reg)^2} \scalarp{\phi(x),h_j}^2 = \sum_{j\in J}\left(\frac{\sigma_j^{(1-\epar)}}{\sigma_j^2+\reg}\right)^2 \frac{\scalarp{\phi(x),h_j}^2}{\sigma_j^{2}} \sigma_j^{2\epar} \\ 
& = \reg^{-(1+\epar)}\sum_{j\in J}\left(\frac{(\sigma_j^2\reg^{-1})^{(1-\epar)/2}}{\sigma_j^2\reg^{-1}+1}\right)^2  \abs{\ell_j(x)}^2\sigma_j^{2\epar}  \leq \reg^{-(1+\epar)}\econ,
\end{align*}
completes the proof.
\end{proof}

Next, we develop concentration bounds of some key quantities used to build RRR and PCR empirical estimators. 

\subsubsection{Variance and Norm of KRR Estimator}\label{app:variance_KRR}

\begin{proposition}\label{prop:krr_norm_bound}
Let \ref{eq:RC}, \ref{eq:SD}  and \ref{eq:KE} hold for some $\rpar\in[1,2]$, $\spar\in(0,1]$ and $\epar\in[\spar,1]$. Given $\delta>0$ if $\rate^1_n(\reg,\delta)<1$, then with probability at least $1-\delta$ in the i.i.d. draw of $(x_i,y_i)_{i=1}^n$ from $\rho$
\[
\PP\left\{ \norm{\Creg^{1/2}(\ERKoop - \RKoop)}\leq \frac{(1+\rcon\, \bcon^{(\rpar-1)/2})\,\rate^2_n(\reg,\delta / 3)}{1-\rate^1_n(\reg,\delta / 3)} \right\} 
\geq 1-\delta,
\]
and 
\[
\PP\left\{ \norm{\ERKoop}\leq 1 + \frac{2\,\rate^3_n(\reg,\delta / 3)}{1-\rate^3_n(\reg,\delta / 3)} \right\} 
\geq 1-\delta.
\]
\end{proposition}
\begin{proof}
Note that $\Creg^{1/2}(\ERKoop - \RKoop) = \Creg^{1/2}(\ECreg^{-1}\ECxy - \Creg^{-1}\Cxy)$, and, hence,
\begin{align}
\Creg^{1/2}(\ERKoop - \RKoop) & = \Creg^{1/2}\ECreg^{-1}(\ECxy - \ECreg\Creg^{-1}\Cxy \pm \Cxy) \nonumber \\
& = \Creg^{1/2}\ECreg^{-1}\Creg^{1/2}\left(\Creg^{-1/2}(\ECxy-\Cxy) - \Creg^{-1/2}(\ECx-\Cx)\Creg^{-1}\Cxy\right). \label{eq:decomp_krr}
\end{align}

Thus, taking the norm and using $\norm{\Creg^{-1}\Cxy}\leq \rcon\, \sigma_1^{\rpar-1}(\TS)$ with the Propositions \ref{prop:bound_left} and \ref{prop:bound_leftright} we prove the first bound. For the second one, we use
\[
\ERKoop - \RKoop = \Creg^{-1}(\ECxy-\Cxy) - \Creg^{-1}(\ECx-\Cx)\ERKoop
\]
with Proposition \ref{prop:bound_left2}  to obtain
\[
\norm{\ERKoop} - 1  \leq\norm{\ERKoop} - \norm{\RKoop}  \leq \norm{\ERKoop - \RKoop} \leq \rate^3_n(\reg,\delta / 2) (1 + \norm{\ERKoop}), 
\]
which completes the proof. 
\end{proof}

\subsubsection{Variance of Singular Values}\label{app:variance_svals}

\begin{proposition}\label{prop:svals_bound}
Let \ref{eq:RC}, \ref{eq:SD}  and \ref{eq:KE} hold for some $\rpar\in[1,2]$, $\spar\in(0,1]$ and $\epar\in[\spar,1]$. Let $B:=\Creg^{-1/2} \Cxy$ and $\EB := \ECreg^{-1/2} \ECxy$. Given $\delta>0$ if $\rate^1_n(\reg,\delta/5)<1$, then with probability at least $1-\delta$ in the i.i.d. draw of $(x_i,y_i)_{i=1}^n$ from $\rho$
\begin{equation}\label{eq:op_B}
\norm{\EB^*\EB -\TB^*\TB} \leq (c^2-1)\,\rate_n(\delta / 5) +  c^2\frac{(\rate^2_n(\reg,\delta / 5))^2}{1-\rate^1_n(\reg,\delta / 3)},    
\end{equation}
where %
$c:=1+\rcon\,\bcon^{(\rpar-1)/2}$. Consequently, for every $i\in[n]$ 
\begin{equation}\label{eq:svals2_bound}
\abs{\sigma_i^2(\EB)-\sigma_i^2(\TB)} \leq (c^2-1)\,\rate_n(\delta / 5) +  c^2\frac{(\rate^2_n(\reg,\delta / 5))^2}{1-\rate^1_n(\reg,\delta / 5)}.
\end{equation}
\end{proposition}
\begin{proof}
We start from the Weyl's inequalities for the square of singular values
\[
\abs{\sigma_i^2(\EB)-\sigma_i^2(\TB)}\leq \norm{\EB^*\EB -\TB^*\TB}, \;i\in[n].
\]
But, since,
\[
\EB^*\EB -\TB^*\TB = \ECxy^*\ECreg^{-1}\ECxy - \Cxy^*\Creg^{-1}\Cxy = (\ECxy-\Cxy)^*\ECreg^{-1}\ECxy + \Cxy^*\Creg^{-1}(\ECxy - \Cxy) +  \Cxy^*(\ECreg^{-1} - \Creg^{-1})\ECxy
\]
denoting $M = \Creg^{-1/2}(\ECxy - \Cxy)$, $N = \Creg^{-1/2}(\ECx - \Cx)$  and  $R:=\Creg^{1/2}(\ERKoop - \RKoop)$, we have
\begin{align*}
\EB^*\EB -\TB^*\TB  & = \TB^*M  + M^*\Creg^{1/2}\ERKoop -  \TB^*N\ERKoop = \TB^*M +  (M^*\Creg^{1/2}-\TB^*N)(\ERKoop \pm \RKoop) \\
& = \TB^*M +  M^*\TB - \TB^*N\RKoop + (M^*-\TB^*N\Creg^{-1/2}) R\\
& = (\RKoop)^*(\ECxy - \Cxy) +  (\ECxy - \Cxy)\RKoop - (\RKoop)^*(\ECx - \Cx)\RKoop + (M^*+(\RKoop)^*N^*) R.
\end{align*}
Therefore, recalling that, due to \eqref{eq:decomp_krr}, $R = \Creg^{1/2}\ECreg^{-1}\Creg^{1/2} (M-N\RKoop)$, we conclude
\begin{align}
    \EB^*\EB -\TB^*\TB = & (\RKoop)^*(\ECxy - \Cxy) +  (\ECxy - \Cxy)\RKoop - (\RKoop)^*(\ECx - \Cx)\RKoop \nonumber \\
    & + (M-N\RKoop)^* \Creg^{1/2}\ECreg^{-1}\Creg^{1/2}(M-N\RKoop). \label{eq:op_B_sq}
\end{align}

Now, applying Propositions \ref{prop:cros_cov_bound}, \ref{prop:bound_leftright} and \ref{prop:bound_left} we obtain  \eqref{eq:op_B}, and, therefore,  \eqref{eq:svals2_bound} follows.
\end{proof}

Remark that to bound singular values we can rely on the fact
\[
\abs{\sigma_i(\EB)-\sigma_i(\TB)} = \frac{\abs{\sigma_i^2(\EB)-\sigma_i^2(\TB)}}{\sigma_i(\EB)+\sigma_i(\TB)} \leq \frac{\abs{\sigma_i^2(\EB)-\sigma_i^2(\TB)}}{ \sigma_i(\EB)\vee \sigma_i(\TB)}.  
\]

\subsubsection{Variance of RRR Estimator}\label{app:variance_RRR}

Recalling the notation $\TB:=\Creg^{-1/2} \Cxy$ and $\EB := \ECreg^{-1/2} \ECxy$, let denote $\TP_r$ and $\EP_r$ denote the orthogonal projector onto the subspace of leading $r$ right singular vectors of $\TB$ and $\EB$, respectively. Then we have $\SVDr{\TB} = \TB\TP_r$ and $\SVDr{\EB} = \EB\EP_r$, and, hence $\RRR = \RKoop\TP_r$ and $\ERRR = \ERKoop\EP_r$.

\begin{proposition}\label{prop:rrr_variance}
Let \ref{eq:RC}, \ref{eq:SD} and \ref{eq:KE} hold for some $\rpar\in[1,2]$, $\spar\in(0,1]$ and  $\epar\in[\spar,1]$. Given $\delta>0$ and $\reg>0$, if $\rate_n^1(\reg,\delta)<1$, then with probability at least $1-\delta$ in the i.i.d. draw of $(x_i,y_i)_{i=1}^n$ from $\rho$,
\begin{equation}\label{eq:variance_bound_rrr}
\norm{\TS(\RRR - \ERRR)} \leq \frac{c\,\rate_n^2(\reg,\delta/5)}{1-\rate_n^1(\reg,\delta/5)}+  \frac{\sigma_1(\TB)}{\sigma_r^2(\TB) -\sigma_{r+1}^2(\TB)} \,\frac{(c^2-1)\,\rate_n(\delta / 5) + c^2\,(\rate_n^2(\reg,\delta/5))^2}{(1-\rate_n^1(\reg,\delta/5))^2},
\end{equation}
where %
$c:=1+\rcon\,\bcon^{(\rpar-1)/2}$
\end{proposition}
\begin{proof}
Start by observing that $\norm{\TS(\RRR - \ERRR)} \leq \norm{\Creg^{1/2}(\RRR - \ERRR)} $ and
\begin{align}
\Creg^{1/2}(\RRR - \ERRR) = & (\Creg^{1/2}\ECreg^{-1}\Creg^{1/2})\cdot \nonumber\\
& \left( \Creg^{-1/2}(\ECx-\Cx)\RRR + \Creg^{-1/2}(\ECxy-\Cxy)\EP_r + \TB (\EP_r-\TP_r)  \right).\label{eq:RRR_aux}
\end{align}
Using that the norm of orthogonal projector $\EP$ is bounded by one and that $\norm{\RRR}\leq\norm{\RKoop}$, applying Propositions \ref{prop:bound_leftright} and \ref{prop:bound_left} together with Propositions \ref{prop:spec_proj_bound} and \ref{prop:svals_bound} completes the proof.
\end{proof}

\subsubsection{Variance of PCR }\label{app:variance_PCR}

Recall that the PCR population estimator is given by $\PCR = \SVDr{ \Creg^{-1}} \Cxy $ while the empirical PCR estimator is $\EPCR = \SVDr{ \ECreg^{-1}} \ECxy $. So, in this case by $\TP_r$ and $\EP_r$ we denote the orthogonal projectors onto the subspace of leading $r$ right singular vectors of $\Cx$ and $\ECx$, respectively. 
Then we have $\PCR=\SVDr{ \Creg^{-1} }\Cxy = \Creg^{-1}\TP_r\Cxy =  \TP_r \Creg^{-1}\Cxy$ and $\EPCR=\SVDr{ \ECreg^{-1} }\ECxy = \ECreg^{-1}\EP_r\ECxy=\EP_r\ECreg^{-1}\ECxy$.

\begin{proposition}\label{prop:pcr_variance}
Let \ref{eq:RC}, \ref{eq:SD} and \ref{eq:KE} hold for some $\rpar\in[1,2]$, $\spar\in(0,1]$ and  $\epar\in[\spar,1]$. Given $\delta>0$ and $\reg>0$, if $\rate_n^1(\reg,\delta)<1$, then for $\Estim = \SVDr{ \Creg^{-1}} \Cxy $ and $\EEstim = \SVDr{ \ECreg^{-1}} \ECxy $ with probability at least $1-\delta$ in the i.i.d. draw of $(x_i,y_i)_{i=1}^n$ from $\rho$,
\begin{equation}\label{eq:variance_bound_pcr}
\norm{\TS(\PCR - \EPCR)} \leq \frac{c\,\rate_n^2(\reg,\delta/4)}{1-\rate_n^1(\reg,\delta/4)} \sqrt{\frac{1+\rate_n^1(\reg,\delta/4)}{1-\rate_n^1(\reg,\delta/4)}}+  \frac{\sigma_1(\TS)}{\sigma_r(\TS) -\sigma_{r+1}(\TS)} \rate_n(\delta / 4),
\end{equation}
where $c:=1+\rcon\,\bcon^{(\rpar-1)/2}$.
\end{proposition}
\begin{proof}
Start by observing that $\norm{\TS(\PCR - \EPCR)} = \norm{\Cx^{1/2}(\PCR - \EPCR)} $ and
\[
\Cx^{1/2}(\PCR - \EPCR) =  \Cx^{1/2}(\TP_r-\EP_r)\RKoop + \Cx^{1/2}\EP_r(\RKoop-\ERKoop).
\]
Therefore, 
\begin{align*}
\norm{\TS(\PCR - \EPCR)} & \leq \norm{\Cx^{1/2}(\TP_r-\EP_r)\RKoop} + \norm{\Creg^{1/2}\EP_r\Creg^{-1/2} \Creg^{1/2}(\RKoop-\ERKoop)}\\
&\leq \sigma_1(\TS)\, \norm{\TP_r-\EP_r} + \norm{\Creg^{1/2}\EP_r\Creg^{-1/2}}\norm{\Creg^{1/2}(\RKoop-\ERKoop)} \\
& \leq \sigma_1(\TS)\, \norm{\TP_r-\EP_r} + \norm{\Creg^{1/2}\ECreg^{-1/2}\EP_r\ECreg^{1/2}\Creg^{-1/2}}\norm{\Creg^{1/2}(\RKoop-\ERKoop)},
\end{align*}
Thus, using
\begin{equation}\label{eq:pcr_var_aux}
\norm{\TS(\PCR - \EPCR)} \leq \sqrt{\bcon} \norm{\TP_r-\EP_r} + \norm{\Creg^{1/2}\ECreg^{-1/2}}\norm{\ECreg^{1/2}\Creg^{-1/2}}\norm{\Creg^{1/2}(\RKoop-\ERKoop)},   
\end{equation}
and applying Propositions \ref{prop:bound_leftright} and \ref{prop:krr_norm_bound} together with Propositions \ref{prop:spec_proj_bound} and \ref{prop:cros_cov_bound} we complete the proof.
\end{proof}

\subsection{Operator Norm Error Bounds}\label{app:error_bound}

Summarising previous sections, in order to prove Theorem \ref{thm:error_bound}, we just need to analyse the bounds $\rate_n^1$, $\rate_n^2$ and $\rate_n^3$. To that end, we use the following result, see Lemma 11 in \cite{Fischer2020}.

\begin{lemma}\label{lem:eff_dim_bound}  
Let \ref{eq:SD} hold for some $\spar\in(0,1]$. Then, if $\spar<1$, for all $\reg>0$
\[
\tr(\Creg^{-1}\Cx) \leq 
\begin{cases}
\frac{\scon^\spar}{1-\spar} \reg^{-\spar} &, \spar<1,\\
\bcon \,\reg^{-1} &, \spar=1.
\end{cases}
\]
\end{lemma}

Now, since $\rpar\geq1$ and $\spar\leq\epar$, we clearly obtain that for large enough $n$ one has 
\begin{equation}\label{eq:rates_asymmp1}
\rate_n^1(\reg,\delta)\lesssim \frac{n^{-1/2}}{\reg^{\epar/2}}\ln\delta^{-1},
\end{equation}
as well as 
\begin{equation}\label{eq:rates_asymmp2}
\rate_n^2(\reg,\delta)\lesssim \left(\frac{n^{-1/2}}{\reg^{\spar/2}} \vee \frac{n^{-1}}{\reg^{\epar/2}}\right)\ln\delta^{-1}\quad\text{ and }\quad\rate_n^3(\reg,\delta)\lesssim \left(\frac{n^{-1/2}}{\reg^{(\spar+1)/2}} \vee \frac{n^{-1}}{\reg^{(\epar+1)/2}}\right)\ln\delta^{-1}.  
\end{equation}

Therefore, as a consequence we have two following results on the estimation of singular values of $\TZ$, and on the operator norm error.

\begin{proposition}\label{prop:svals_rate}
Let \ref{eq:RC} and \ref{eq:SD} hold for some $\rpar\in[1,2]$ and $\spar\in(0,1]$.Then, there exists a constant $c>0$ such that for every given $\delta\in(0,1)$, large enough $n>r$ and small enough $\reg>0$ with probability at least $1-\delta$ in the i.i.d. draw of $(x_i,y_i)_{i=1}^n$ from $\rho$
\begin{equation}\label{eq:svals_rate}
\abs{\sigma_j(\EB) - \sigma_j(\Koop\TS)} \lesssim n^{-\frac{\rpar}{2(\rpar+2\spar)}} \ln\delta^{-1}.
\end{equation}
\end{proposition}
\begin{proof}
The proof is direct consequence of Propositions \ref{prop:svals_app_bound} and \ref{prop:svals_bound} using \eqref{eq:rates_asymmp1}-\eqref{eq:rates_asymmp2}.
\end{proof}

\thmError*
\begin{proof}
Since $\sigma_r(\TZ) > \sigma_{r+1}(\TZ)$ implies due to Proposition \ref{prop:svals_app_bound} that for small enough $\reg>0$ we have that $\sigma_r(\TB) > \sigma_{r+1}(\TB)$,  Propositions \ref{prop:app_bound}, \ref{prop:rrr_bias} and \ref{prop:rrr_variance}, ensure that for large enough $n$ \eqref{eq:error_bound_rrr} holds.

Similarly, Propositions \ref{prop:app_bound}, \ref{prop:pcr_bias} and \ref{prop:pcr_variance}, assure that for large enough $n$ \eqref{eq:error_bound_pcr} holds if $\sigma_r(\TS) > \sigma_{r+1}(\TS)$.
\end{proof}

Setting $\reg_n = n^{-\frac{1}{\rpar+\spar}}$ we have that 
\[
\reg_n^{\rpar/2} = \frac{n^{-1/2}}{\reg_n^{\spar/2}} = n^{-\frac{\rpar}{2 (\rpar+\spar)}},\quad \frac{n^{-1/2}}{\reg_n^{\epar/2}} = n^{-\frac{\rpar+\spar-\epar}{2 (\rpar+\spar)}}\quad\text{ and }\quad  \frac{n^{-1/2}}{\reg_n^{(1+\spar)/2}} = n^{-\frac{\rpar-1}{2( \rpar+\spar)}},
\]
and, hence, recalling that $\rpar\geq1$, 
\[
 \lim_{n\to\infty}\rate^2_n(\reg_n,\delta / 5) = \lim_{n\to\infty}\rate^1_n(\reg_n,\delta / 5) =\lim_{n\to\infty}\rate_n(\delta / 5) = 0,
\]
while $\rate^3_n(\reg_n,\delta / 5)$ is, in the worst case, bounded.

Therefore, for this choice of regularization parameter Theorem \ref{thm:error_bound} assures that
\begin{equation}\label{eq:error_rate}
\error(\ERRR) - \sigma_{r+1}(\TB) \leq \error(\ERRR) - \sigma_{r+1}(\TZ) \lesssim n^{-\frac{\rpar}{2 (\rpar+\spar)}}.    
\end{equation}

Using the same arguments as the above, the following lower bounds follow.

\begin{theorem}\label{thm:error_bound_conc} 
Under the assumptions of Theorem \ref{thm:error_bound}, recalling \eqref{eq:opt_reg},
there exists a constant $c>0$ such that for large enough $n\geq r$ and every $i\in[r]$ with probability at least $1-\delta$ in the i.i.d. draw of %
$\Data$ from $\rho$
\[
\error(\ERRR) \geq  \sigma_{r+1}(\Koop\TS) - c\,\rate_n^\star\,\ln \delta^{-1},
\]
and
\[
\error(\EPCR) \geq  \sigma_{r+1}(\TS) - \sqrt{\IrRisk^{r+1}} - c\,\rate_n^\star\,\ln \delta^{-1},
\]
where $\IrRisk^{r+1}:= \scalarp{(\TS\TS^* - \Koop\TS\TS^*\Koop^*)\ell_{r+1},\ell_{r+1}}\geq0$ is the irreducible risk restricted to the $(r+1)$-st left singular function $\ell_{r+1}$ of $\TS$.
\end{theorem}

The Example \ref{ex:OU} in the main body, that we now elaborate, shows that lower bounds of the previous theorem are tight.

\begin{example}[Example\ref{ex:OU}]\label{ex:OU_app}
Consider the 1D equidistant sampling of the Ornstein–Uhlenbeck process, obtained by integrating Langevin equation of Example~\ref{ex:langevin} with $\beta = 1$ and $U(x) = x^{2}/2$
\[X_{t} = e^{-1}X_{t-1} + \sqrt{1-e^{-2}}\,\epsilon_t\]
where $x \in \R =: \X$ and $\{\epsilon_t\}_{t\geq 1}$ are independent standard Gaussians. For this process it is well-known that $\im$ is $\mathcal{N}(0,1)$ and that $\Koop$ admits a spectral decomposition $(\keval_i, \kefun_i)_{i \in \N}$ in terms of Hermite polynomials. Namely, $\{\keval_j = e^{-(j-1)}\}_{j\geq 1}$ with corresponding orthonormal eigenbasis $\kefun_j(x)= \frac{1}{\sqrt{(j-1)!}}H_{j-1}(x)$ where $H_j(x)= \partial_1^j G(0,x)$, for $G(s,x)= e^{sx-\frac{1}{2}s^2}$, denotes the $j$-th probabilists' Hermite polynomial. See e.g. \cite{klus2018data} for more details.

We now make use of the spectral decomposition of $\Koop$ to design the following class of kernel functions
\begin{equation*}
    k_{\Pi, \nu}(x, x^{\prime}) := \sum_{i \in N} \keval_{\Pi(i)}^{2\nu}\kefun_{i}(x)\kefun_i(x^{\prime}),
\end{equation*}
where $\Pi$ is a permutation of the indices of the eigenvalues and $\nu$ is a scaling factor. The rationale behind these kernels is to morph the original metric structure of $\Koop$ in a way which is harder and harder to revert when learning from finite sets of data. In particular, for any target rank $r$, we set $\nu:=1/r^2$ and $\Pi$ to the permutation such that $i \mapsto 2r - i + 1$ ($i \geq r$), $i \mapsto i - r$ ($r + 1 \leq i \leq 2r$) and $i \mapsto i$ elsewhere.  Then, we immediately have $\IrRisk^{r+1}:= \scalarp{(\TS\TS^* - \TZ\TZ^*)\ell_{r+1},\ell_{r+1}} = \keval_{r+1}^{2\nu} -\keval_{r+1}^{2\nu}\keval_{1}^2 = 0$, while $\sigma_{r+1}(\TS)= \sqrt{\keval_{r+1}^{2\nu}} = e^{-\nu r} = e^{-1/r}$. On the other hand, $\sigma_{r+1}(\TZ) = \sqrt{\keval_{r+1}^2\keval_{1}^{2\nu}} = e^{-r}$. In view of Theorem~\ref{thm:error_bound_conc}, observing that in this example we can set $\spar\to0$, 
\[
\abs{\error(\EPCR) - e^{- 1/r}} \lesssim n^{-1/2}\ln\delta^{-1},
\]
and
\[
\abs{\error(\ERRR) - e^{-r}}\lesssim n^{-1/2}\ln\delta^{-1}.
\]
\end{example}

\subsection{Minimax optimal operator norm bounds}\label{app:missspec_optimal}

In this section we first derive a minimax lower bound for the operator norm convergence rate of the Koopman operator. This approach follows a standard framework already used for instance in \cite{Li2022,Fischer2020} to derive a lower bound on the Hilbert-Schmidt norm convergence rate of the Conditional Mean Embedding. We typically assume (as in \cite{Li2022,Fischer2020}) that condition \ref{eq:SD} is sharp, i.e. there exists $\spar\in(0,1]$ and a constant $\scon>1$ such that for every $j\in J$ 
\begin{equation}
\label{eq:SD2} 
\scon^{-1}\,j^{-1/\spar} \leq \eval_j(\Cx)\leq\scon\,j^{-1/\spar}.
\end{equation}
Then we prove that the learning rate of the RRR estimator is sharp (up to a log factor).

\begin{theorem}
\label{thm:min_lower_bound}
\VK{Let $0<\beta\leq \epar\leq 1$ and $\alpha\in (0,2]$} be such that conditions \ref{eq:SD}, \ref{eq:RC} and \ref{eq:KE} hold. Then for any $r\geq2$, there exist absolute constants %
$c,\,q>0$ such that for all learning methods $\Data \rightarrow \EEstim \in \HSr$, $\delta\in (0,1)$, and all sufficiently large $n\geq 1$, there is a distribution $\rho$ on $\X\times\X$ used to sample $\Data$, with marginal distribution $\pi$ on $\X$, such that with probability at least %
$1-\delta$,
$$
\error(\EEstim)\geq %
c\,\delta^q\,n^{-\frac{\VK{\rpar \vee \epar}}{2(\VK{(\rpar \vee \epar)}+\spar)}}.
$$
\end{theorem}
\begin{proof}
We remind first that condition \ref{eq:SRC} is more restrictive than condition \ref{eq:RC} (see Remark \ref{rem:src_implies_rc}). This means that a minimax lower under Condition \ref{eq:SRC} also holds under less restrictive Condition \ref{eq:RC}. We remark next that a Koopman operator is fully characterized by the conditional 
distribution of $X_{t+1}\vert X_{t}$. Therefore we will consider the set of well-separated conditional distributions on $\X$ introduced in \cite{LGBPP12,Li2022}: $\{p_j(\cdot \vert x),\, j\in [M]\}$. Hence the corresponding set of Koopman operators $\Omega = \left\lbrace A_j , j \in [M] \right\rbrace$ admits rank at most $2$ and when restricted to the RKHS $\RKHS$ satisfies $%
\CME(x) = \mathbb{E}[\phi(X_{t+1})\vert X_{t}=x] = \int_{\X}  p_j(dy\vert x ) \phi(y)$. 

Fix now $a\in \X$ such that $\norm{\phi(a)}\geq \bcon/2$ (such $a$ always exists under condition \ref{eq:KE}). The boundedness of the kernel also guarantees that
\begin{align*}
\int \langle p_j(dy\vert x ) -  p_{j'}(dy\vert x ), \phi(a) \rangle^2 \pi(dx) \leq \norm{\phi(a)}^2\norm{(A_j-A_{j'})\TS}^2  \leq \bcon \norm{(A_j-A_{j'})\TS}^2.
\end{align*}
Hence it is sufficient to lower bound the left-hand-side to 
guarantee the existence of well-separated hypothesis. It follows from  Lemma 7 in \cite{Li2022} that $M\geq 2^{c\,n^{\frac{\spar}{(\VK{(\rpar \vee \epar)}+\spar)}}}$ and 
$$
\frac{1}{c\,\bcon} n^{-\frac{\VK{\rpar \vee \epar}}{\VK{(\rpar \vee \epar)}+\spar}} \leq \norm{(A_j-A_{j'})\TS}^2 \leq  c n^{-\frac{\VK{\rpar \vee \epar}}{\VK{(\rpar \vee \epar)}+\spar}} ,
$$
for some absolute constant $c>1$. Then, we can apply Lemma 8 in \cite{Li2022} to get the result.
\end{proof}

This result states that, under conditions \ref{eq:SD}, \ref{eq:RC} and \ref{eq:KE}, for $\rpar\geq1$ no estimator can achieve a learning rate faster than  $ n^{-\frac{\alpha}{2(\alpha+\beta)}}$ in the operator norm. Note that this lower bound is matching our upper bound on the variance term (up to $\ln\delta^{-1}$) in \eqref{eq:error_bound_rrr} and \eqref{eq:error_bound_pcr} of the RRR and PCR estimators respectively. We note however that a significant difference between these bounds concerns the "bias" term which is equal to $\sigma_{r+1}(\Koop\TS)$ for RRR %
and $\sigma_{r+1}(\TS)$ for PCR. 

\VK{Hence, to have an optimal learning rate for PCR, due to \eqref{eq:SD}, we need to choose $r\geq n^{\frac{\rpar\spar}{\rpar+\spar}}-1$. On the other hand, assuming} for instance that $\Koop\TS$ admits finite rank $r$, then the bias disappears completely in \eqref{eq:error_bound_rrr} and consequently the RRR estimator is minimax optimal (up to a log factor). Moreover, even if the Koopman operator is infinite-dimensional with e.g. eigenvalues decaying exponentially fast, then %
the bias term is negligible in front of the variance term. Thus RRR is still minimax (up to a log).

\subsection{Extension to misspecified setting}
\label{sec:extensionSRC}
\VK{Next, recalling results from \cite{Li2022} on HS-norm error of the KRR estimator, we extend our results for RRR and PCR beyond the case \ref{eq:RC} for $\rpar\in[1,2]$. Namely, in the reminder of this section we will assume \ref{eq:SRC} instead of \ref{eq:RC}, i.e. we will assume that $\HKoop^\alpha$ is not only bounded by also a HS-operator when $\alpha<1$.}

\VK{
\begin{theorem}\label{thm:krr_misspec} 
Let  \ref{eq:SRC}, \ref{eq:SD} and \ref{eq:KE} hold for some $\rpar\in(0,2]$, $\spar\in(0,1]$ and $\epar\in[\spar,1]$, respectively, and let $\cl(\range(\TS))=\Lii$. Let 
\begin{equation}\label{eq:opt_reg_misspec1}
    \reg\asymp n^{-\frac{1}{\rpar+\spar}}\,\text{ and }\,\rate^\star_n:= n^{-\frac{\rpar}{2(\rpar+\spar)}}\quad \text{ when }\; \rpar \geq \epar-\spar,
\end{equation}
and for some $s>1$
\begin{equation}\label{eq:opt_reg_misspec2}
    \reg\asymp \big(\tfrac{n}{\ln^s n}\big)^{-\frac{1}{\epar}}\,\text{ and }\,\rate^\star_n:= \big(\tfrac{n}{\ln^s n}\big)^{-\frac{\rpar}{2\epar}}\quad \text{ when }\; \rpar < \epar-\spar.
\end{equation}
Then, for every $\delta\in(0,1)$ there exists a constant $c\,{>}\,0$, depending only on $\RKHS$, such that for large enough $n\geq r$, with probability at least $1\,{-}\,\delta$ in the i.i.d. draw of $\Data$ from $\rho$
\[
\error(\ERKoop)\leq c\,\rate_n^\star \,\ln \delta^{-1}.
\]
\end{theorem}}
\begin{proof}
\VK{Recalling that $\error(\ERKoop)=\norm{\Koop\TS-\TS \ERKoop}\leq \hnorm{\Koop\TS-\TS \ERKoop}$, this result is a direct consequence of \cite[Theorem 2]{Li2022}.} 

\VK{In particular, recalling notation
denoting $M = \Creg^{-1/2}(\ECxy - \Cxy)$ and $N = \Creg^{-1/2}(\ECx - \Cx)$ we have that 
\begin{equation}\label{eq:missepc_variance}
\hnorm{M-N\RKoop}\leq \widetilde{\rate}^2_n(\reg,\delta)\quad\text{ w.p. at least } 1-\delta  
\end{equation}
where
\begin{equation}\label{eq:alt_rate2}
\widetilde{\rate}^2_n(\reg,\delta):= 4\,\sqrt{2}\,\ln\frac{2}{\delta}\,\sqrt{\frac{\bcon\tr(\Creg^{-1}\Cx)}{n} +  \frac{4\bcon\econ}{n^2\reg^{\epar}} + \frac{\rcon^2 \econ}{n} \reg^{\rpar-\epar}},
\end{equation}
and, since $R=\Creg^{1/2}(\ERKoop - \RKoop) = \Creg^{1/2}\ECreg^{-1}\Creg^{1/2} (M-N\RKoop)$, consequently, 
\[
\PP\left\{ \hnorm{\Creg^{1/2}(\ERKoop - \RKoop)}\leq \frac{\widetilde{\rate}^2_n(\reg,\delta/2)}{1-\rate^1_n(\reg,\delta / 2)} \right\} 
\geq 1-\delta.
\]
Combining this with the approximation bound of Proposition \ref{prop:app_bound}, we obtain the result.} 
\end{proof}

\VK{
Therefore, when $\rpar>\epar-\spar$ previous theorem shows that the operator norm error of KRR estimator is optimal. It is important to note that regime $\rpar\leq\epar-\spar$ the optimal rates are an open problem in operator learning as well as in the classical regression setting. We conclude this section by extending Theorem \ref{thm:error_bound} and showing that PCR and RRR also achieve optimal rates when $\rpar>\epar-\spar$.}

\VK{
\begin{theorem}\label{thm:error_bound_misspec} 
Assume the operator $\Koop$ satisfies $\sigma_r(\Koop\TS)>\sigma_{r+1}(\Koop\TS)\geq0$ for some $r\in\N$. Let  \ref{eq:SRC}, \ref{eq:SD} and \ref{eq:KE} hold for some $\rpar\in(0,2]$, $\spar\in(0,1]$ and $\epar\in[\spar,1]$, respectively, and let $\cl(\range(\TS))=\Lii$. Let $\reg$ and $\rate^\star_n$ be given by \eqref{eq:opt_reg_misspec1}-\eqref{eq:opt_reg_misspec2} and let $\delta\in(0,1)$. Then, there exists a constant $c\,{>}\,0$, depending only on $\RKHS$, such that for large enough $n\geq r$, with probability at least $1\,{-}\,\delta$ in the i.i.d. draw of 
$\Data$ from $\rho$
\begin{subnumcases}{\error(\EEstim)\leq}
   \sigma_{r+1}(\Koop\TS){+}c\,\rate_n^\star\,\ln \delta^{-1} & {\rm if} \,$\EEstim=\ERRR$, \vspace{.25truecm} \label{eq:error_bound_rrr_misspec}
   \\
   \sigma_{r+1}(\TS) +  c\,\rate_n^\star\,\ln \delta^{-1} & {\rm if}\, $\EEstim = \EPCR \,\,{\rm and\,\,} \sigma_r(\TS)>\sigma_{r+1}(\TS)$. \label{eq:error_bound_pcr_misspec}
\end{subnumcases}
\end{theorem}
\begin{proof}
Proof of \eqref{eq:error_bound_pcr_misspec} readily follows from \eqref{eq:pcr_var_aux} by applying \eqref{eq:missepc_variance} instead of Proposition \ref{prop:krr_norm_bound}, i.e. using $\widetilde{\rate}_n^2$ given in \eqref{eq:alt_rate2} instead of $\rate_n^2$ given in \eqref{eq:reg_eta2}. To show \eqref{eq:error_bound_rrr_misspec}, recalling \eqref{eq:op_B_sq},  note that
\[
\EB^*\EB -\TB^*\TB = \TB^*(M-N\RKoop) + M^*\TB + (M-N\RKoop)^* \Creg^{1/2}\ECreg^{-1}\Creg^{1/2} (M-N\RKoop) 
\]
we have that w.p. $1-\delta$
\[
\norm{\EB^*\EB -\TB^*\TB}\leq \sqrt{\bcon}\big[\widetilde{\rate}_n^2(\reg,\delta/3) + {\rate}_n^2(\reg,\delta/3)\big] + \frac{[\widetilde{\rate}_n^2(\reg,\delta/3)]^2}{1-{\rate}_n^1(\reg,\delta/3)}\lesssim \widetilde{\rate}_n^2(\reg,\delta/3).
\]
Therefore, since 
\[
\Creg^{1/2}(\RRR - \ERRR) = \Creg^{1/2}[\ERKoop-\RKoop] \EP_r + \TB[\EP_r-\TP_r],
\]
we conclude that w.p. $1-\delta$
\[
\norm{\TS(\RRR - \ERRR)} \lesssim \widetilde{\rate}_n^2(\reg,\delta/3) \left(1+  \frac{\sigma_1(\TB)}{\sigma_r^2(\TB) -\sigma_{r+1}^2(\TB)} \right),
\]
which after balancing with the approximation bound of Proposition \ref{prop:app_bound} yields the final result.
\end{proof}
}

\section{Spectral Learning Rates}\label{app:spec_learning} 

We now prove the statements of Theorems \ref{thm:spectral_uniform_main} and \ref{thm:spectral_main} in slightly more general form.

\begin{theorem}[RRR]\label{thm:spectral_rrr}
Let $\Koop$ be a compact self-adjoint operator. Under the assumptions of Theorem \ref{thm:spectral_uniform_main}, if $\EEstim=\ERRR$, then there exists a constant $c>0$ (depending only on the RKHS) such that for every $\delta\in(0,1)$,  every large enough $n\geq r$ and every $i\in[r]$ with probability at least $1-\delta$ in the i.i.d. draw of $\Data$ from $\rho$
\begin{equation}\label{eq:bound_eval_rrr_ii_app}
\abs{\eeval_i-\keval_{{j(i)}}} \leq \frac{\sigma_{r+1}(\TZ)}{\sigma_{r}(\TZ)} \,\left( 1 + c\,n^{-\frac{\rpar-1}{2 (\rpar+\spar)}}\,\ln\delta^{-1} \right) + c\,\,n^{-\frac{\rpar}{2 (\rpar+\spar)}}\,\ln\delta^{-1}.
\end{equation}
Moreover, 
\begin{equation}\label{eq:bound_eval_rrr_i_app}
\abs{\eeval_i-\keval_{{j(i)}}} \leq \emetdist_i\,\sigma_{r+1}(\EB) + c\,n^{-\frac{\rpar}{2 (\rpar+\spar)}}\ln\delta^{-1},
\end{equation}
and
\begin{equation}\label{eq:bound_efun_rrr_i_app}
\norm{\ekefun_{i} - \kefun_{{j(i)}}}^2\leq \frac{2(\emetdist_i\,\sigma_{r+1}(\EB) + c\,n^{-\frac{\rpar}{2 (\rpar+\spar)}}\ln\delta^{-1}) }{[\gap_{{j(i)}}(\Koop) - \emetdist_i\sigma_{r+1}(\EB) - c\,n^{-\frac{\rpar}{2 (\rpar+\spar)}}\ln\delta^{-1}]_{+}}.
\end{equation}
\end{theorem}
\begin{proof}
First, observe that Proposition \ref{prop:metric_dist} and Weyl's inequality imply that for every $i\in[r]$
\[
\metdist(\erefun_i) \leq \frac{(\abs{\eeval_i}\cond(\eeval_i))\,\wedge\,\norm{\ERRR}}{[\sigma_{r}(\TS\RRR)-\norm{\TS(\ERRR-\RRR)}]_+}\leq \frac{(\abs{\eeval_i}\cond(\eeval_i))\,\wedge\,\norm{\ERKoop}}{[\sigma_{r}(\TS\RRR)-\norm{\TS(\ERRR-\RRR)}]_+}.
\]
But, 
\[
\sigma_{r}(\TS\RRR) \geq \sigma_{r}(\Creg^{1/2}\RRR)-\sqrt{\reg}\norm{\RRR} \geq \sigma_{r}(\TB)-\sqrt{\reg}\norm{\ERKoop},
\]
and, using Proposition \ref{prop:svals_app_bound} and $\norm{\RKoop}\leq1$ we conclude 
\[
\metdist(\erefun_i) \leq \frac{(\abs{\eeval_i}\cond(\eeval_i))\,\wedge\,\norm{\ERKoop}}{[\sigma_{r}(\TZ) - \rcon \norm{\Cx}^{\rpar/4}\reg^{\rpar/4} - \sqrt{\reg} -\norm{\TS(\ERRR-\RRR)}]_+}.
\]

Hence, in view of Propositions \ref{prop:krr_norm_bound} and \ref{prop:rrr_variance}, %
we conclude that for some constant $c'>0$
\[
\metdist(\erefun_i) \leq \frac{ (\abs{\eeval_i}\cond(\eeval_i))\,\wedge\,(1 + c'\,n^{-\frac{\rpar-1}{2(\rpar+\spar)}}\ln\delta^{-1})}{[\sigma_{r}(\TZ) - c'\,n^{-\frac{\rpar}{4(\rpar+\spar)}}\ln\delta^{-1}]_+}
\]
and, consequently, for some $c>0$,
\[
\metdist(\erefun_i) \leq 
2\frac{ \abs{\eeval_i}\cond(\eeval_i)}{\sigma_{r}(\TZ)}
\,\wedge\, \left( \frac{2}{\sigma_{r}(\TZ)} + c\, n^{-\frac{\rpar-1}{2(\rpar+\spar)}}\ln\delta^{-1}\right).
\]

Therefore, using \eqref{eq:error_rate}, \eqref{eq:bound_eval_rrr_ii_app} directly follows.

Next, since we have obtain that $\metdist(\erefun_i)$ is bounded, due to Proposition \ref{prop:emp_met_dist}, empirical metric distortions $\emetdist_i$ are bounded, too.

So, assuming that $\sigma_{r+1}(\TZ)>0$, we have that
\begin{align*}
\metdist(\erefun_i)\,\error(\ERRR) &\leq \metdist(\erefun_i)\,\sigma_{r+1}(\TB) + \metdist(\erefun_i)\left(\error(\ERRR) - \sigma_{r+1}(\TB)\right) \\
& \leq \emetdist_i\,\sigma_{r+1}(\TB) + \emetdist_i \metdist^2(\erefun_i) \rate_n(\delta/3) + \metdist(\erefun_i)\left(\error(\ERRR) - \sigma_{r+1}(\TB)\right)\\
&\leq \emetdist_i\,\sigma_{r+1}(\EB) + \emetdist_i\,  \frac{\abs{\sigma_{r+1}^2(\EB)-\sigma_{r+1}^2(\TB)}}{ \sigma_{r+1}(\TB)} + \emetdist_i \metdist^2(\erefun_i) \rate_n(\delta/3) \\
& \quad + \metdist(\erefun_i)\left(\error(\ERRR) - \sigma_{r+1}(\TB)\right).
\end{align*}
Thus, recalling Proposition \ref{prop:svals_bound}, and using \eqref{eq:error_rate} we obtain that there exists a constant $c>0$ depending only on the RKHS so that
\[
\metdist(\erefun_i)\,\error(\ERRR) \leq \emetdist_i\,\sigma_{r+1}(\EB) + c\,n^{-\frac{\rpar}{2 (\rpar+\spar)}}\ln\delta^{-1}.
\]
Finally, applying Theorem \ref{thm:spectral_perturbation} concludes the proof.
\end{proof}

Specifying the previous result to finite rank Koopman operators we obtain the following.

\begin{corollary}[RRR]\label{cor:finite_rank_koop_rrr}
If $\Koop$ is of finite rank $r\in\N$, under the assumptions of Theorem \ref{thm:spectral_rrr}, with probability at least $1-\delta$ in the i.i.d. draw of $(x_i,y_i)_{i=1}^n$ from $\rho$
\begin{equation}\label{eq:bound_eval_rrr_i_app_finiterank}
\abs{\eeval_i-\keval_{{j(i)}}} \leq c\,n^{-\frac{\rpar}{2 (\rpar+\spar)}}\ln\delta^{-1},
\end{equation}
and 
\begin{equation}\label{eq:bound_efun_rrr_i_app_finiterank}
\norm{\ekefun_{i} - \kefun_{{j(i)}}}^2\leq \frac{2\, c\,n^{-\frac{\rpar}{2 (\rpar+\spar)}}\ln\delta^{-1}}{[\gap_{i}(\ERRR) - 3\,c\,n^{-\frac{\rpar}{2 (\rpar+\spar)}}\ln\delta^{-1}]_{+}}.
\end{equation}
\end{corollary}
\begin{proof}
First, observe that in the previous proof we could conclude the same having $\sigma_{r+1}(\TZ)$ instead of $\sigma_{r+1}(\EB)$. But, since rank of $\TZ$ is at most $r$, we conclude \eqref{eq:bound_eval_rrr_i_app_finiterank}. Next, 
since the ranks of $\ERRR$ and $\Koop$ are equal, we have the same number of nonzero eigenvalues. Hence, in view of the eigenvalue ordering $\eeval_1\geq\eeval_2\geq\ldots \eeval_r$ and $\keval_{j_1}\geq\keval_{j_2}\geq\ldots \keval_{j_r}$, we obtain 
 \[
\gap_i(\ERRR) = \abs{\eeval_{i} - \eeval_{i-1}} \wedge \abs{\eeval_{i}-\eeval_{i+1}} \leq \gap_i(\Koop) + \abs{\eeval_{i} - \keval_{{j(i)}}} + (\abs{\eeval_{i-1} - \keval_{j_{i-1}}} \vee \abs{\eeval_{i+1}-\keval_{j_{i+1}}}).
\]
So, applying \eqref{eq:bound_eval_rrr_i_app_finiterank}, we conclude \eqref{eq:bound_efun_rrr_i_app_finiterank}.
\end{proof}

In analogous way we have the following theorem for PCR estimator. 
\begin{theorem}[PCR]\label{thm:spectral_pcr}
Let $\Koop$ be a compact self-adjoint operator. Under the assumptions of Theorem \ref{thm:spectral_uniform_main}, if $\EEstim=\EPCR$, then there exists a constant $c>0$ (depending only on the RKHS) such that for every $\delta\in(0,1)$,  every large enough $n\geq r$ and every $i\in[r]$ with probability at least $1-\delta$ in the i.i.d. draw of $\Data$ from $\rho$
\begin{equation}\label{eq:bound_eval_pcr_ii_app}
\abs{\eeval_i-\keval_{{j(i)}}} \leq \frac{\sigma_{r+1}(\TS)}{[\sigma_{r}(\TZ)-\sigma_{r+1}^\rpar(\TS)]_+} \,\left( 2 + c\,n^{-\frac{\rpar-1}{2 (\rpar+\spar)}}\,\ln\delta^{-1} \right) + c\,n^{-\frac{\rpar}{2 (\rpar+\spar)}}\,\ln\delta^{-1}.
\end{equation}
Moreover, 
\begin{equation}\label{eq:bound_eval_pcr_i_app}
\abs{\eeval_i-\keval_{{j(i)}}} \leq \emetdist_i\,\sigma_{r+1}(\ES) + c\,n^{-\frac{\rpar}{2 (\rpar+\spar)}}\ln\delta^{-1},
\end{equation}
and 
\begin{equation}\label{eq:bound_efun_pcr_i_app}
\norm{\ekefun_{i} - \kefun_{{j(i)}}}^2\leq \frac{2(\emetdist_i\,\sigma_{r+1}(\ES) + c\,n^{-\frac{\rpar}{2 (\rpar+\spar)}}\ln\delta^{-1})}{[\gap_{{j(i)}}(\Koop) - \emetdist_i\sigma_{r+1}(\ES) - c\,n^{-\frac{\rpar}{2 (\rpar+\spar)}}\ln\delta^{-1}]_{+}}.
\end{equation}
\end{theorem}
\begin{proof}
The proof follows along the same lines of reasoning as for the RRR estimator, with the exception that we use 
\[
\sigma_r(\TS\PCR) = \sigma_r(\TP_r\Cx^{1/2}\RKoop) \geq \sigma_r(\Cx^{1/2}\RKoop) - \norm{(I-\TP_r)\Cx^{1/2}\Creg^{-1}\Cxy} \geq \sigma_r(\TB) - \sqrt{\reg} - \sigma_{r+1}^\rpar(\TS),
\]
where $\TP_r$ denotes the projector onto a subspace of $r$ leading eigenfunctions of $\Cx$.
\end{proof}

We finally remark that form the proofs of Theorems \ref{thm:spectral_pcr} and \ref{thm:spectral_pcr}, recalling notation of the uniform bias in Eq.~\eqref{eq:unif_bias}, we also have
\begin{equation}\label{eq:bound_eval_iii_app}
\frac{\abs{\eeval_i-\keval_{{j(i)}}}}{\abs{\eeval_i}\cond(\eeval_i)} \leq s + c\,n^{-\frac{\rpar}{2 (\rpar+\spar)}}\ln\delta^{-1},
\end{equation}
which may also be of interest.

\VK{
\begin{remark}[Limitation to \ref{eq:RC} for $\rpar\geq1$]
Note that, according to Proposition \ref{prop:app_bound}, $\RKoop$ may become unbounded as $\reg\to0$ for $\alpha<1$. This, recalling Proposition \ref{prop:metric_dist}, implies that the metric distortion of estimated eigenfunctions may deteriorate. So, proving tight spectral rates in the misspecified regime remains an interesting open problem.  
\end{remark}
}

\section{Experiments}\label{app:exp} 

\textbf{A Realistic Example: Langevin Dynamics.} 
We here append some additional data on the Langevin dynamics experiment not fitting in the main body. We recall that the eigenpairs of $\Koop$ in this experiment have a tangible physical interpretation~\cite{Schwantes2015}. Indeed $\keval_i$ is related to the typical time scale needed for a particle to cross one of the potential barriers, while $\kefun_i$ is approximately constant in the regions of the phase space where the particle spends a lot of time (metastable states) and have sharp variations near the unstable points of the dynamics (transition states)

In Figure~\ref{fig:langevin_dynamics} we show the eigenvalue and eigenfunction errors, and fitted the lines appearing in the log-log plots to get the decay rates for the eigenvalue and eigenfunctions errors reported in Table~\ref{tab:eval} and \ref{tab:efun}. 

\begin{figure}[t!]
    \centering
\includegraphics[width=0.45\textwidth]{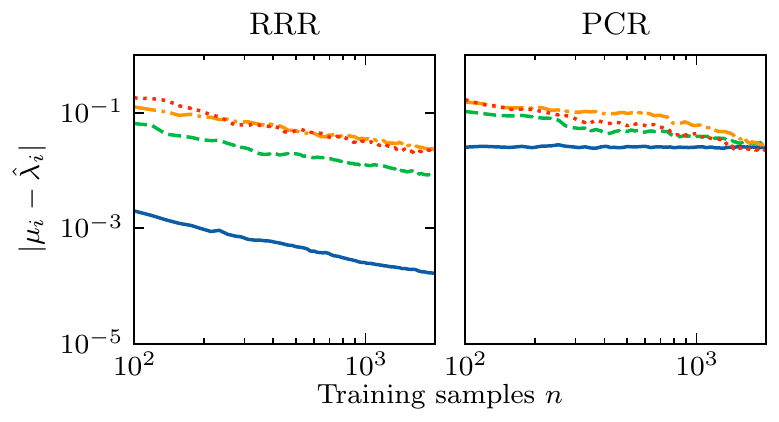}
\includegraphics[width=0.45\textwidth]{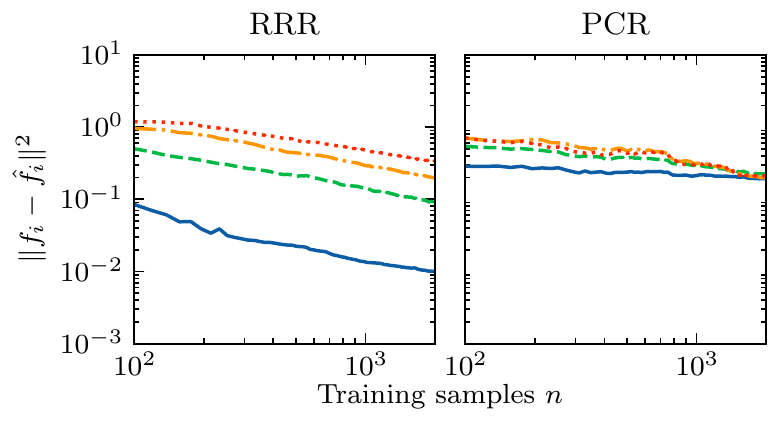}
\vspace{-.3truecm}
    \caption{Eigenvalue errors (left pair) and Eigenfunction errors (right pair)~vs.~sample size. 
Solid, dashed, dash-dotted, and dotted lines denote the error on the four leading eigenpairs, in order. %
Both axes are in log scale, so straight lines correspond to a polynomial decay of the error with rate the line slope. RRR estimator is unbiased and attains smaller errors than the biased PCR one.}
\label{fig:langevin_dynamics}
\end{figure}

\begin{table}[!h]
\centering
\begin{tabular}{lcc}
\toprule
 & RRR & PCR \\
 \midrule
$\eeval_0$ & -0.801699 & -0.009406 \\
$\eeval_1$ & -0.617175 & -0.410670 \\
$\eeval_2$ & -0.553720 & -0.631005 \\
$\eeval_3$ & -0.699849 & -0.697323 \\
\bottomrule
\end{tabular}
\caption{Fitted decay rates for the estimated Koopman eigenvalues for the Langevin dynamics experiment}\label{tab:eval}
\end{table}

\begin{table}[!h]
\centering
\begin{tabular}{lcc}
\toprule
 & RRR & PCR \\
 \midrule
$\ekefun_0$ & -0.635680 & -0.139309 \\
$\ekefun_1$ & -0.577666 & -0.326457 \\
$\ekefun_2$ & -0.573375 & -0.505743 \\
$\ekefun_3$ & -0.476403 & -0.450845 \\
\bottomrule
\end{tabular}
\caption{Fitted decay rates for the estimated Koopman eigenfunctions for the Langevin dynamics experiment.}
\label{tab:efun}
\end{table}

\textbf{Model selection and the Alanine dipeptide dataset}
We use a simulation of the small molecule Alanine dipeptide reported in Ref.~\cite{Wehmeyer2018}. For each RRR estimator we set the rank $r=5$ and the Tikhonov regularization~$\reg = 10^{-6}$. The 19 different kernels are the following: 7 RBF kernels with length scales $\sigma \in \{0.05, 0.1, 0.15, 0.2, 0.25, 0.3, 0.35 \}$, and 12 Mat\'ern kernels corresponding to each possible combination of $\nu \in \{1.5, 2.5\}$ and length scale $\sigma \in \{0.05, 0.1, 0.15, 0.2, 0.25, 0.3\}$. The optimal kernel turned out to be RBF with $\sigma = 0.35$. We report forecasting RMSE of $30 = 3*10$ positions of the 10 atoms of the system.

\end{document}